%% file: nips14_combined_distributed_KL_clean.tex
\documentclass{article} %
\usepackage{nips14submit_e,times}
\usepackage{hyperref}
\usepackage{url}

\usepackage{wrapfig}

\usepackage{makecell}

\usepackage[comma,authoryear]{natbib}
\usepackage{amsmath, amsfonts, amssymb, amsthm}
\usepackage{times}		%
\usepackage{parskip}		%
\usepackage[pdftex]{graphicx}	%
\usepackage[raggedright]{sidecap}
\usepackage{wrapfig} 
\usepackage{hyperref}	%
\usepackage{multirow} \usepackage{rotating}
\usepackage{color}

\usepackage{pgfplots}
\usetikzlibrary{plotmarks}
\usepackage{scalefnt}%

\hypersetup{
  colorlinks,
  citecolor=blue,
  linkcolor=blue,
  urlcolor=blue}
  
\usepackage{algorithm}
\usepackage{algorithmic}

\usepackage{MnSymbol, graphicx}

\usepackage{cutwin}
\usepackage{lipsum}

\input{newcommands.tex}

\title{Distributed Estimation, Information Loss and Exponential Families}%

\author{
Qiang Liu \hspace{20mm} Alexander Ihler \\
Department of Computer Science, University of California, Irvine\\
\texttt{qliu1@uci.edu} \hspace{6mm} \texttt{ihler@ics.uci.edu}\\
}

\nipsfinalcopy %

\begin{document}

\maketitle

\begin{abstract}
Distributed learning of probabilistic models from multiple data repositories with minimum communication is increasingly important. 
We study a simple communication-efficient learning framework that first calculates the local maximum likelihood estimates (MLE) based on the data subsets, and then combines the local MLEs to achieve the best possible approximation to the global
MLE given the whole dataset.
We study this framework's statistical properties, showing that 
the 
efficiency loss
compared to the global setting relates to 
how much the underlying distribution families deviate from full exponential families, drawing connection to the theory of information loss by Fisher, Rao and Efron. 
We show that the ``full-exponential-family-ness" represents the lower bound of the error rate of arbitrary combinations of local MLEs, and is achieved by a KL-divergence-based combination method but not by a more common linear combination method. 
We also study the empirical properties of 
both
methods, showing that the KL method significantly outperforms linear combination in practical settings with issues such as model misspecification, non-convexity, and heterogeneous data partitions. 
\end{abstract}

\section{Introduction}

Modern data-science applications increasingly require distributed learning
algorithms to extract information from  many data repositories stored at
different locations with minimal interaction.  Such distributed settings are
created due to high communication costs (for example in sensor networks), or
privacy and ownership issues (such as sensitive medical or financial data).
Traditional algorithms often require access to the entire dataset
simultaneously, and are not suitable for distributed settings. 

We consider a straightforward two-step procedure for distributed learning that follows a ``divide and conquer" strategy: %
(i) local learning, which involves learning probabilistic models based on the local data repositories separately, and 
(ii) model combination, where the local models are transmitted to a central node (the ``fusion center''), and combined to form a global model that integrates the information in the local repositories. 
This framework only requires transmitting the local model parameters to the fusion center once, yielding significant advantages in terms of both communication and privacy constraints. 
However, the two-step procedure may not fully extract all the information in the data, and may be less (statistically) efficient than a corresponding centralized learning algorithm that operates globally on the whole dataset. 
This raises important challenges in understanding the fundamental statistical limits of the local learning framework, and proposing optimal combination methods to best approximate the global learning algorithm.

In this work, we study these problems in the setting of estimating generative model parameters from a distribution family via the maximum likelihood estimator (MLE). 
We show that the loss of statistical efficiency caused by using the local learning framework is related to how much the underlying distribution families deviate from full exponential families: local learning can be as efficient as (in fact exactly equivalent to) global learning on full exponential families, but is 
 less efficient on non-exponential families, depending on how nearly ``full exponential family" they are. 
 The ``full-exponential-family-ness" is formally captured by the \emph{statistical curvature} originally defined by  \citet{efron1975defining}, and is a measure of the minimum loss of Fisher information when summarizing the data using first order efficient estimators \citep[e.g.,][]{fisher1925theory, rao1963criteria}. 
Specifically, we show that arbitrary combinations of the local MLEs on the local datasets can approximate the global MLE on the whole dataset at most up to an asymptotic error rate proportional to the square of the statistical curvature. 
In addition, a KL-divergence-based combination of the local MLEs achieves this minimum error rate in general, and exactly recovers the global MLE on full exponential families. 
In contrast, a more widely-used linear combination method does not achieve the optimal error rate, and makes mistakes even on full exponential families. 
We also study the two methods empirically,
examining their robustness against practical issues such as model mis-specification, heterogeneous data partitions, and the existence of hidden variables (e.g., in the Gaussian mixture model). 
These issues often cause the likelihood to have multiple local optima, and can easily degrade the linear combination method. On the other hand, the KL method remains robust in these practical settings.

\textbf{Related Work.}
Our work is related to \citet{zhang2013jmlr}, which includes a theoretical analysis for linear combination. 
\citet{merugu2003privacy, merugu2006distributed} proposed the KL combination method in the setting of Gaussian mixtures, but without theoretical analysis. 
There are many recent theoretical works on distributed learning \citep[e.g.,][]{predd2007distributed, balcan2012distributed, zhang2013information, shamir2013fundamental}, but most focus on discrimination tasks like classification and regression. 
There are also many works on distributed clustering \citep[e.g.,][]{merugu2003privacy, forero2011distributed, liangdistributed} and distributed MCMC %
 \citep[e.g.,][]{scott2013bayes, wang2013parallel, neiswanger2013asymptotically}. 
 An orthogonal setting of distributed learning is when the data is split across the variable dimensions, instead of the data instances; see e.g.,  \citet[][]{liu12a, meng13marginal}. 
\section{Problem Setting}
\label{sec:setting}
Assume we have an i.i.d.\ sample $X = \{ x^i ~\colon~ i = 1, \ldots, n\}$, partitioned into $d$ sub-samples $X^{k} =\{ x^{i} ~\colon~ i \in \alpha_k \}$ that are stored in different locations, 
where $\cup_{k=1}^d \alpha_k = [n]$. For simplicity, we assume the data are equally partitioned, so that each group has $n/d$ instances; extensions to the more general case is straightforward. 
Assume $X$ is drawn i.i.d.\ from a distribution with an unknown density from a distribution family $\{ \density{x}{\theta} \colon \theta \in \Theta \}$. Let $\theta^*$ be the true unknown parameter. We are interested in estimating $\thetatrue$ via the maximum likelihood estimator (MLE) based on the whole sample, 
$$
\thetamle = \argmax_{\theta \in \Theta} \sum_{i\in [n]}\log \density{x^i}{ \theta}. 
$$
However, directly calculating the global MLE often requires distributed optimization algorithms (such as ADMM \citep{boyd2011distributed}) that need iterative communication between the local repositories and the fusion center, which can significantly slow down the algorithm regardless of the amount of information communicated at each iteration. 
We instead approximate the global MLE by a two-stage procedure that 
calculates the local MLEs separately 
for each sub-sample,
then sends the local MLEs to 
the fusion center and combines them. 
Specifically, the $k$-th sub-sample's local MLE is
$$
\thetak = \argmax_{\theta \in \Theta}  \sum_{i\in \alpha^k}\log \density{x^i}{ \theta}, 
$$
and we want to construct a combination function $f(\thetahat_1, \ldots, \thetahat_d) \to \thetaf$ to form the best approximation to the global MLE $\thetamle$. %
Perhaps the most straightforward combination is the linear average,
\begin{flalign*}
&\hspace{0.25\textwidth} \text{\emph{Linear-Averaging}:~~~~~}\thetalinear = \frac{1}{d} \sum_k \thetak. &%
\end{flalign*}
However, this method is obviously limited to continuous and additive parameters; in the sequel, we illustrate
it also tends to degenerate in the presence of practical issues such as non-convexity and non-i.i.d.\ data partitions. 
A better combination method is to average the \emph{models} w.r.t.\ some distance metric, instead of the \emph{parameters}. In particular, we consider a KL-divergence based averaging, %
\begin{flalign}
&\hspace{0.25\textwidth}\text{\emph{KL-Averaging}: ~~~~~}  \thetaKL = \argmin_{\theta\in \Theta} \sum_k \KL(p(x | \thetak) ~|| ~p(x | \theta)). &
\label{equ:thetakl}
\end{flalign}
The estimate $\thetakl$ can also be motivated by a parametric bootstrap procedure that first draws sample $X^{k'}$ from each local model $p(x | \thetak)$, and then estimates a global MLE based on all the combined bootstrap samples  $X' = \{X^{k'}\colon k \in [d]\}$.  We can readily show that this reduces to $\thetakl$ as the size of the bootstrapped samples $X^{k'}$ grows to infinity. Other combination methods based on different distance metrics are also possible, but may not have a similarly natural interpretation.

\section{Exactness on Full Exponential Families}
In this section, we analyze the KL and linear combination methods on full exponential families. We show that the KL combination of the local MLEs exactly equals the global MLE, while the linear average does not 
in general, but can be made exact by using a special parameterization. This suggests that distributed learning is in some sense ``easy" on full exponential families. 

\begin{defn}
(1). A family of distributions is said to be a full exponential family if its density can be represented in a canonical form (up to one-to-one transforms of the parameters), 
\begin{align*}
p(x | \theta) = \exp(\theta^T \phi(x)  - \logZ(\theta)), &&
\theta \in \Theta \equiv  \{\theta \in \R^m \colon \int_x \exp(\theta^T \phi(x)) d H(x) < \infty  \}. 
\end{align*}
where $\theta = [\theta_1, \ldots \theta_{m}]^T$ and  $\phi(x) = [\phi_1(x), \ldots \phi_m(x)]^T$ are called the natural parameters and the natural sufficient statistics, respectively. 
The quantity $Z(\theta)$ is the normalization constant, and $H(x)$ is the reference measure. %
An exponential family is said to be \emph{minimal} if $[1, \phi_1(x), \ldots \phi_m(x)]^T$ is linearly independent, that is, there is no non-zero constant vector $\alpha$, such that $\alpha^T\phi(x)=0$ for all $ x$. 
\end{defn}

\begin{thm}
\label{thm:exactexp}
If $ \mathcal{P}= \{p(x | {\theta}) \colon \theta \in \Theta\}$ is a full exponential family, then the KL-average $\thetakl$ always exactly recovers the global MLE, that is, $\thetakl = \thetamle$. 
Further, if $\mathcal{P}$ is minimal, we have
\begin{align}
\label{equ:mumean}
\thetakl =  \mu^{-1} \left ( \frac{ \mu(\hat{\theta}^1) + \cdots +  \mu(\hat{\theta}^d)}{d} \right ), %
\end{align}
where $\mu~ \colon~ \theta \mapsto \E_{\theta}[\phi(x)]$ is the one-to-one map from the natural parameters to the moment parameters, and $\mu^{-1}$ is the inverse map of $\mu$. Note that  we have $\mu(\theta) = {\partial\logZ(\theta)}/{\partial\theta}$. 
\end{thm}
\begin{proof}
Directly verify that the KL objective in \eqref{equ:thetakl} equals the global negative log-likelihood. %
\end{proof}
The nonlinear average in \eqref{equ:mumean} gives an intuitive interpretation of why $\thetakl$ equals $\thetamle$ on full exponential families: it first calculates the local empirical moment parameters $\mu(\thetak) = d/n\sum_{i\in \alpha^k} \phi(x^k)$; averaging them gives the empirical moment parameter on the whole data $\hat\mu_n = 1/n\sum_{i\in [n]} \phi(x^k)$, which then exactly maps to the global MLE. 

Eq~\eqref{equ:mumean} also suggests that $\thetalinear$ would be exact only if $\mu(\cdot)$ is an identity map. 
Therefore, one may make $\thetalinear$ exact by using the special parameterization $\vartheta = \mu(\theta)$. In contrast, KL-averaging will make this reparameterization automatically ($\mu$ is different on different exponential families). 
Note that both KL-averaging and global MLE are invariant w.r.t. one-to-one transforms of the parameter $\theta$, but linear averaging is not. %
\begin{exa}[Variance Estimation]
Consider estimating the variance $\sigma^2$ of a zero-mean Gaussian distribution. 
Let $\hat s_k = (d/n)  \sum_{i\in \alpha^k} (x^i)^2 $ be the empirical variance on the $k$-th sub-sample and $\hat s = \sum_k \hat s_k/d$ the overall empirical variance. 
Then, $\thetalinear$ would correspond to different power means on $\hat s_k$, 
depending on the choice of parameterization, e.g., 
 \renewcommand{\arraystretch}{2} 
 \begin{table}[h!]
 \centering
  \begin{tabular}{|l|c|c|c|}
 \hline
 \text{} & $\theta = \sigma^2$ {(variance)}   &   $\theta = \sigma$ \text{(standard deviation)}  &  $\theta = \sigma^{-2}$ \text{(precision)}  \\ \hline
 $\thetalinear$  & $ \frac{1}{d}\sum_k\hat s_k$ &   $\frac{1}{d}\sum_k ({\hat s_k})^{1/2}$    & $\frac{1}{d}\sum_k(\hat s_k)^{-1}$  \\  \hline
\end{tabular}
\end{table} \\
where only the linear average of $\hat s_k$ (when $\theta=\sigma^2$) matches the overall empirical variance $\hat s$ and equals the global MLE. 
In contrast,  $\thetakl$ always corresponds to a linear average of $\hat s_k$, equaling the global MLE, regardless of the parameterization. 
\end{exa}

\section{Information Loss in Distributed Learning}
\label{sec:asymp}
The exactness of $\thetakl$ in Theorem~\ref{thm:exactexp} is due to the beauty (or simplicity) of exponential families. 
Following Efron's intuition, full exponential families can be viewed as ``straight lines" or ``linear subspaces" in the space of distributions,  while other distribution families correspond to ``curved" sets of distributions, whose deviation from full exponential families can be measured by their \emph{statistical curvatures} as defined by  \citet{efron1975defining}.  
That work
shows that statistical curvature is closely related to Fisher and Rao's theory of second order efficiency \citep{fisher1925theory, rao1963criteria}, and represents the minimum information loss when summarizing the data using first order efficient estimators. %
In this section, we connect this classical theory with the local learning framework, and show that the statistical curvature also represents the minimum asymptotic deviation of arbitrary combinations of the local MLEs to the global MLE, and that this is achieved by the KL combination method, but not in general by the simpler linear combination method.  %

\subsection{Curved Exponential Families and Statistical Curvature}

We follow the convention in \citet{efron1975defining}, and illustrate the idea of statistical curvature using \emph{curved} exponential families, which are smooth sub-families of full exponential families. 
The theory can be naturally extended to more general families \citep[see e.g.,][]{efron1975defining, kass2011geometrical}. 
\begin{defn}\label{def:curved}
A family of distributions $\{p(x|{\theta}) \colon \theta \in \Theta\}$ is said to be a curved exponential family if its density can be represented as
 \begin{align}
 p(x | \theta) = \exp(\eta(\theta)^T \phi(x) - \logZ(\eta(\theta))), 
 \label{equ:defcurved}
 \end{align}
 where the dimension of $\theta = [\theta_1, \ldots, \theta_\dimtheta]$ is assumed to be smaller than that of $\eta = [\eta_1, \ldots, \eta_{\dimeta}]$ and $\phi = [\phi_1, \ldots, \phi_\dimeta]$, that is $\dimtheta < \dimeta$. 

Following \citet{kass2011geometrical}, we assume some regularity conditions for our 
asymptotic analysis. Assume $\Theta$ is an open set in $\R^{\dimtheta}$, and the mapping $\eta~\colon~ \Theta \to \eta(\Theta)$ is one-to-one and infinitely differentiable, and of rank $\dimtheta$, meaning that the $\dimtheta\times\dimeta$ matrix $\dot\eta(\theta)$ has rank $\dimtheta$ everywhere. In addition, if a sequence $\{\eta(\theta_i) \in N_0\}$ converges to a point $\eta(\theta_0)$, then $\{\eta_i \in \Theta\}$ must converge to $\phi(\eta_0)$. In geometric terminology, such a map $\eta~\colon~ \Theta \to \eta(\Theta)$ is called a $q$-dimensional \emph{embedding} in $\R^{\dimeta}$.   
\end{defn}

Obviously, a curved exponential family can be treated as a smooth subset of a full exponential family $p(x | \eta) = \exp(\eta^T \phi(x) - \logZ(\eta))$, with $\eta$ constrained in $\eta(\Theta)$.  If $\eta(\theta)$ is a linear function, then the curved exponential family can be rewritten into a full exponential family in lower dimensions;
otherwise, $\eta(\theta)$ is a curved subset in the $\eta$-space, whose curvature -- its deviation from planes or straight lines --   represents its deviation from full exponential families. 

\begin{wrapfigure}{r}{0.23\textwidth} \centering
\vspace{-1.5em}
    \includegraphics[width=0.2\textwidth]{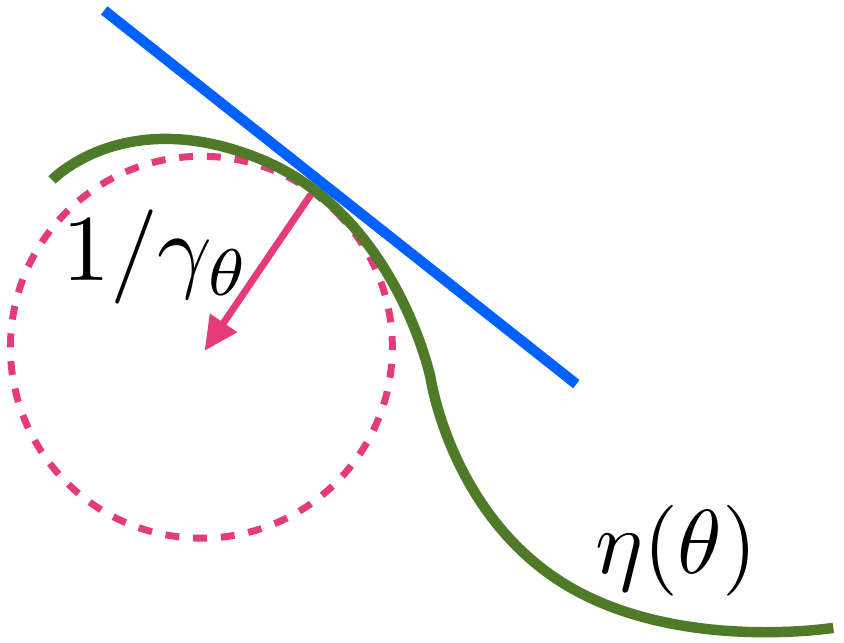}
\vspace{-1em}  
\end{wrapfigure}

Consider the case when $\theta$ is a scalar, and hence $\eta(\theta)$ is a curve; the geometric curvature $\gamma_{\theta}$ of $\eta(\theta)$ at point $\theta$ is defined to be the reciprocal of the radius of the circle that fits best to $\eta(\theta)$ locally at $\theta$. Therefore, the curvature of a circle of radius $r$ is a constant $1/r$.  
In general, elementary calculus shows that 
$\gamma_\theta^2 =  (  \dot\eta_{\theta}^T \dot\eta_{\theta} )^{-3} ( \ddot\eta_{\theta}^T \ddot\eta_{\theta}  \cdot  \dot\eta_{\theta}^T \dot\eta_{\theta}  - (  \ddot\eta_{\theta}^T \dot\eta_{\theta})^2) $. 
The \emph{statistical curvature} of a curved exponential family is defined similarly, except equipped with an inner product defined via its Fisher information metric.  
\begin{defn}[Statistical Curvature]
Consider a curved exponential family $\mathcal{P} = \{p(x|\theta) \colon \theta \in \Theta \}$, %
whose parameter $\theta$ is a scalar ($\dimtheta=1$). Let $\Sigma_{\theta} = \cov_{\theta} [\phi(x)]$ be the $\dimeta\times \dimeta$ Fisher information on the corresponding full exponential family $p(x | \eta)$.   
The statistical curvature of $\mathcal{P}$ at $\theta$ is defined as
\begin{align*}
\gamma_{\theta}^2 =  (  \dot\eta_{\theta}^T\Sigma_{\theta}  \dot\eta_{\theta} )^{-3} \big [ (\ddot\eta_{\theta}^T \Sigma_{\theta} \ddot\eta_{\theta})  \cdot (\dot\eta_{\theta}^T\Sigma_{\theta} \dot\eta_{\theta})  - (  \ddot\eta_{\theta}^T \Sigma_{\theta} \dot\eta_{\theta})^2 \big ]. 
\end{align*}
\end{defn}
The definition can be extended to general multi-dimensional parameters, but requires involved notation. 
We give the full definition and our general results in the appendix. %
\begin{exa}[Bivariate Normal on Ellipse] 
\label{exa:ellipseexample}
Consider a bivariate normal distribution with diagonal covariance matrix and mean vector restricted on an ellipse $\eta(\theta) = [a \cos(\theta),  b\sin(\theta)]$, that is, 
$$p(x | \theta)\propto \exp \big [-\frac{1}{2}(x_1^2+x_2^2) +  a\cos\theta~x_1  +  b \sin\theta ~x_2) \big ], ~~~~ \theta \in (-\pi, \pi), ~~ x \in \mathbb{R}^2.$$ 
We have that $\Sigma_\theta$ equals the identity matrix in this case, and the statistical curvature equals the geometric curvature of the ellipse in the Euclidian space, 
$
\gamma_{\theta} = {ab}  (a^2\sin^2(\theta) + b^2 \cos^2 (\theta))^{-3/2}. %
$
\end{exa}

The statistical curvature was originally defined by \citet{efron1975defining} as the minimum amount of information loss when summarizing the sample using first order efficient estimators. \citet{efron1975defining} showed that, extending the result of \citet{fisher1925theory} and \citet{rao1963criteria}, 
\newcommand{\ttrue}{\thetatrue}
\begin{align}
\label{equ:infloss}
\lim_{n\to \infty}[ \Fisher_{\ttrue}^X  - \Fisher^{\thetamle}_{\ttrue}] =   \gamma_{\thetatrue}^2 \fisher_{\thetatrue},\end{align}
where $I_{\thetatrue}$ is the Fisher information (per data instance) of the distribution $p(x | \theta)$ at the true parameter $\thetatrue$, and $\Fisher_{\thetatrue}^X = n \fisher_{\thetatrue}$ is the total information included in a sample $X$ of size $n$, and $\Fisher^{\thetamle}_{\thetatrue}$ is the Fisher information included in $\thetamle$ based on $X$. Intuitively speaking, we lose about $\gamma_{\thetatrue}^2$ units of Fisher information when summarizing the data using the ML estimator. 
\citet{fisher1925theory} also interpreted $\gamma_{\thetatrue}^2$ as the effective number of data instances lost in MLE, easily seen from rewriting $ \Fisher^{\thetamle}_{\thetatrue} \approx (n - \gamma_{\thetatrue}^2) I_{\thetatrue}$, as compared to $\Fisher_{\thetatrue}^X = n \fisher_{\thetatrue}$. 
Moreover, this is the minimum possible information loss in the class of ``first order efficient" estimators $T(X)$, those which satisfy the weaker condition $\lim_{n\to \infty} \Fisher_{\thetatrue} / \Fisher^{T}_{\thetatrue} = 1.$ 
Rao coined the term ``second order efficiency" for this property of the MLE. 

The intuition here has direct implications for our distributed setting, 
since $\thetaf$ depends on the data only through $\{\thetak\}$, each of which summarizes the data with a loss of $\gamma_{\thetatrue}^2$ units of information. The total information loss is $d \cdot \gamma_{\thetatrue}^2 $, in contrast with the global MLE, which only loses $ \gamma_{\thetatrue}^2 $ overall.  Therefore, the additional loss due to the distributed setting is $(d-1) \cdot \gamma_{\thetatrue}^2 $. 
We will see that our results in the sequel closely match this intuition. 
\subsection{Lower Bound}
\label{sec:tojoint}
The extra information loss $(d-1)\gamma_{\thetatrue}^2$ turns out to be the asymptotic lower bound of the mean square error rate $n^2 \E_{\thetatrue}[ I_{\thetatrue}(\thetaf - \thetamle)^2]$ for any arbitrary combination function $f(\thetahat^1, \ldots,\thetahat^d)$. 
\begin{thm}[Lower Bound]
\label{thm:lowerbound}
For an arbitrary measurable function $\thetaf \!=\! f(\thetahat^1, \ldots, \thetahat^d)$, we have
$$
\liminf_{n\to +\infty}~ n^2 ~ \E_{\thetatrue}[|| f(\thetahat^1, \ldots, \thetahat^d)  - \thetamle ||^2 ] \geq
(d-1)\gamma_{\thetatrue}^2 I_{\thetatrue}^{-1}.  %
$$
\end{thm}
\begin{proof}[Sketch of Proof ]
Note that
\begin{align*}
\E_{\thetatrue}[|| \thetaf  - \thetamle ||^2 ]  
&= \E_{\thetatrue}[|| \thetaf  - \E_{\thetatrue}(\thetamle | \thetahat^1,\ldots,\thetahat^d) ||^2 ]  + \E_{\thetatrue}[|| \thetamle - \E_{\thetatrue}(\thetamle | \thetahat^1,\ldots,\thetahat^d)  ||^2 ]  \\
&\geq  \E_{\thetatrue}[|| \thetamle - \E_{\thetatrue}(\thetamle | \thetahat^1,\ldots,\thetahat^d)  ||^2]   \\
& = \E_{\thetatrue}[ \var_{\thetatrue}(\thetamle | \thetahat^1,\ldots,\thetahat^d)], 
\end{align*}
where the lower bound is achieved when $\thetaf  = \E_{\thetatrue}(\thetamle | \thetahat^1,\ldots,\thetahat^d)$.  
The conclusion follows by showing that $\lim_{n\to +\infty}\E_{\thetatrue}[ \var_{\thetatrue}(\thetamle | \thetahat^1,\ldots,\thetahat^d)]  = (d-1)\gamma_{\thetatrue}^2 I_{\thetatrue}^{-1}$; 
 this requires involved asymptotic analysis, and is presented in the Appendix. %
\end{proof}

\begin{wrapfigure}{r}{0.3\textwidth}
\vspace{-2em}
  \begin{center}
    \includegraphics[width=0.3\textwidth]{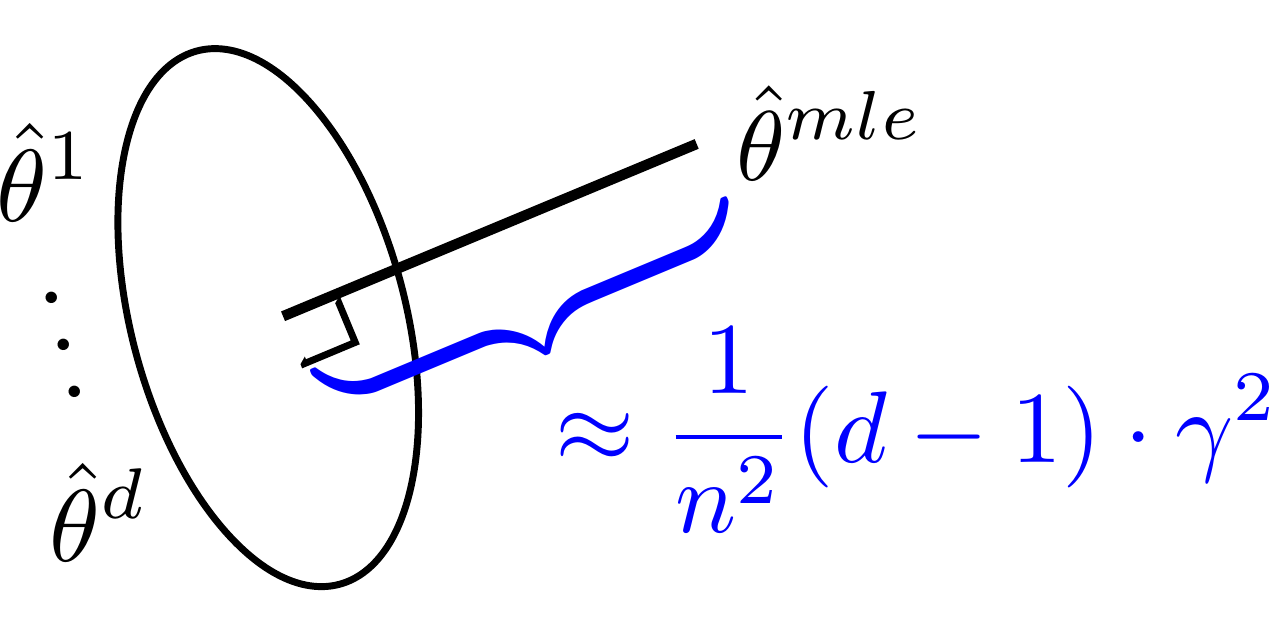}
  \end{center}
\vspace{-1.5em}
\end{wrapfigure}
The proof above highlights a geometric interpretation via the projection of random variables \citep[e.g., ][]{van2000asymptotic}. Let $\mathcal{F}$ be the set of all random variables in the form of $f(\thetahat^1, \ldots,\thetahat^d)$. The optimal consensus function should be the projection of $\thetamle$ onto $\mathcal{F}$, and the minimum mean square error is the distance between $\thetamle$ and $\mathcal{F}$. %
The conditional expectation $\thetaf  = \E_{\thetatrue}(\thetamle | \thetahat^1,\ldots,\thetahat^d)$ is the exact projection and  ideally the best combination function; however, this is intractable to calculate due to the dependence on the unknown true parameter $\thetatrue$. We show in the sequel that $\thetakl$ gives an efficient approximation and achieves the same asymptotic lower bound. 

\subsection{General Consistent Combination}
We now analyze the performance of a general class of $\thetaf$, which includes both the KL average $\thetaKL$ and the linear average $\thetalinear$; we show that $\thetaKL$ matches the lower bound in Theorem~\ref{thm:lowerbound}, while $\thetalinear$ is not optimal even on full exponential families. 
We start by defining conditions which any ``reasonable" $f(\thetahat^1, \ldots, \thetahat^d)$ should satisfy. 

\begin{defn}
\label{def:consistent}
(1). We say $f(\cdot)$ is \emph{consistent}, if for $\forall \theta \in \Theta$, $\theta^k \to \theta$, 
$\forall k\in [d] $  %
implies 
$f(\theta^1, \ldots, \theta^d) \to 
\theta$.%
(2). %
$f(\cdot)$ is \emph{symmetric} if 
$ f(\thetahat^1, \ldots, \thetahat^d)  =  f(\thetahat^{\sigma(1)}, \ldots, \thetahat^{\sigma(d)}), $
for any permutation $\sigma $ on $[d]$.    
\end{defn}
The consistency condition guarantees that if all the $\thetak$ are consistent estimators, then $\thetaf$ should also be consistent. The symmetry is also straightforward due to the symmetry of the data partition $\{X^k\}$. In fact, if $f(\cdot)$ is not symmetric, one can always construct a symmetric version that performs better or at least the same (see Appendix for details). 
We are now ready to present the main result. 

\begin{thm}
\label{thm:main}
(1). Consider a consistent and symmetric $\thetaf =  f(\thetahat^1, \ldots, \thetahat^d)$ as in Definition~\ref{def:consistent}, 
whose
first three orders of derivatives exist. %
Then, for curved exponential families in Definition~\ref{def:curved}, 
\begin{align*}
&\Etrue [ \thetaf  - \thetamle ] ~=~ \frac{d-1}{n} \blue{\beta_{\thetatrue}^f}  + \oo(n^{-1}),  \\
&\Etrue [ ||\thetaf  - \thetamle||^2 ] ~= ~ \frac{d-1}{n^2} \cdot [{\gamma_{\thetatrue}^2}{I_{\thetatrue}^{-1}} + \blue{ (d+1)(\beta_{\thetatrue}^f)^2} ]   + \oo(n^{-2}), 
\end{align*}
where $\beta_{\thetatrue}^{f}$ is a term that depends on the choice of the combination function $f(\cdot)$. 
Note that the mean square error is consistent with the lower bound in Theorem~\ref{thm:lowerbound}, and is tight if $\beta_{\thetatrue}^{f} = 0$. 

(2). The KL average $\thetakl$ has $\beta_{\thetatrue}^{f} = 0$, and hence achieves the minimum bias and mean square error, 
\begin{align*}
\Etrue [ \thetaKL - \thetamle ] ~=~  \oo(n^{-1}), && 
\Etrue[ ||\thetaKL  - \thetamle||^2 ] ~= ~ \frac{d-1}{n^2} \cdot {\gamma_{\thetatrue}^2}{I_{\thetatrue}^{-1}}   + \oo(n^{-2}). 
\end{align*}
In particular, note that the bias of $\thetakl$ is smaller in magnitude than that of general $\thetaf$ with $\beta_{\thetatrue}^f\neq0$. 
(4). The linear averaging $\thetalinear$, however, does not achieve the lower bound in general. We have
$$
\beta_{\thetatrue}^{\linear} = 
I_*^{-2}( \ddot\eta_{\thetatrue}^T\Sigma_{\thetatrue} \dot\eta_{\thetatrue}  
+\frac{1}{2} \bigg.  \E_{\thetatrue} \big[  \frac{\partial^3\log p(x | \thetatrue)}{\partial \theta^3}  \big]), 
$$
which is in general non-zero even for full exponential families. %

(5). The MSE w.r.t. the global MLE $\thetamle$ can be related to the MSE w.r.t. the true parameter $\thetatrue$, by
\begin{align*}
&\Etrue [ ||\thetaKL  - \thetatrue||^2 ]  ~ = ~\Etrue [ ||\thetamle  - \thetatrue||^2 ]  ~ +~  \frac{d-1}{n^2} \cdot {\gamma_{\thetatrue}^2}{I_{\thetatrue}^{-1}}   + \oo(n^{-2}). \\
&\Etrue [ ||\thetalinear  - \thetatrue||^2 ]  ~= ~ \Etrue [ ||\thetamle  - \thetatrue||^2 ]   ~+ ~ \frac{d-1}{n^2} \cdot [  {\gamma_{\thetatrue}^2}{I_{\thetatrue}^{-1}} + 2(\beta^{\linear}_{\thetatrue})^2  ] + \oo(n^{-2}).
\end{align*}
\end{thm}
\begin{proof}
See Appendix for the proof and the general results for multi-dimensional parameters. 
\end{proof}
Theorem~\ref{thm:main} suggests that $\thetaf  - \thetamle = \Op(1/n)$ for any consistent $f(\cdot)$, which is smaller in magnitude than $\thetamle - \thetatrue = \Op(1/\sqrt{n})$. 
Therefore, any consistent $\thetaf$ is first order efficient, in that its difference from the global MLE $\thetamle$ is negligible compared to $\thetamle - \thetatrue$ asymptotically. 
This also suggests that KL and the linear methods perform roughly the same asymptotically in terms of recovering the true parameter $\thetatrue$. 
However, we need to treat this claim with caution, because, as we demonstrate empirically, the linear method may significantly degenerate 
 in the non-asymptotic region or when the conditions in Theorem~\ref{thm:main} do not hold. 

\section{Experiments and Practical Issues}
We present numerical experiments 
to demonstrate 
the correctness of %
our theoretical analysis. 
More importantly, we also
 study empirical properties of the linear and KL combination methods 
that are not enlightened by the asymptotic analysis. %
We find that the linear average 
tends to degrade significantly when its local models ($\thetak$) are not already close, 
for example due to small sample sizes, 
heterogenous data partitions, or 
non-convex likelihoods
(so that different local models find different local optima). 
In contrast, the KL combination 
is much more robust in practice. 

\subsection{Bivariate Normal on Ellipse}
We start with the toy model in Example~\ref{exa:ellipseexample} 
to verify our theoretical results. 
We draw samples from the true model (assuming $\thetatrue = \pi/4$, $a = 1$, $b = 5$), and partition the samples randomly into 10 sub-groups ($d=10$). 
		Fig.~\ref{fig:ellipse} shows that the empirical biases and MSEs match closely with the theoretical predictions when the sample size is large (e.g., $n \geq 250$), 
		and $\thetakl$ is consistently better than $\thetalinear$ in terms of recovering both the global MLE and the true parameters. 
		Fig.~\ref{fig:ellipse}(b) shows that the bias of $\thetakl$ decreases faster than that of $\thetalinear$, as predicted in Theorem~\ref{thm:main}~(2). 
		Fig.~\ref{fig:ellipse}(c) shows that all algorithms perform similarly in terms of the asymptotic MSE w.r.t. the true parameters $\thetatrue$, but linear average degrades significantly in the non-asymptotic region (e.g., $n<250$). 

		\begin{wrapfigure}{r}{0.25\textwidth}
		\vspace{-1.5em}
			\begin{center}
				\includegraphics[width=0.25\textwidth]{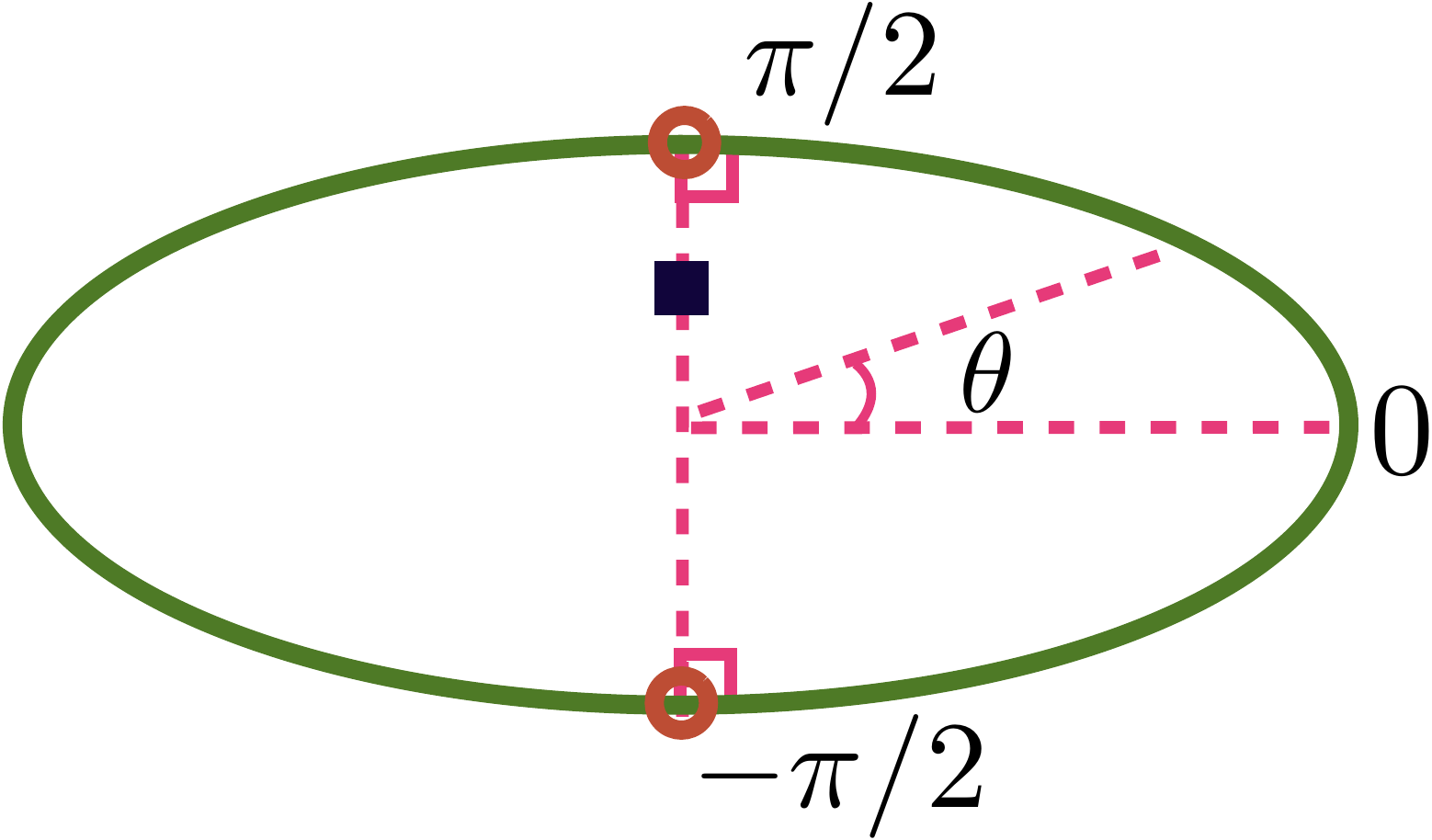}%
			\end{center}
		\vspace{-1em}  
		\end{wrapfigure}
		\textbf{Model Misspecification.} 
		Model misspecification is unavoidable in practice, and may create multiple local modes in the likelihood objective, leading to poor behavior from the linear average. 
		We illustrate this phenomenon using the toy model in Example~\ref{exa:ellipseexample}, assuming the true model is $\normal([0, 1/2],~  \mathbf{1}_{2\times2})$, outside of the assumed parametric family. This is illustrated in the figure at right, where the ellipse represents the parametric family, and the black square denotes the true model. The MLE will concentrate on the projection of the true model to the ellipse, in one of two locations ($\theta = \pm\pi/2$) indicated by the two red circles. 
		Depending on the random data sample, the global MLE will concentrate on one or the other of these two values;
		see Fig.~\ref{fig:ellipsemisspecify}(a).
		Given a sufficient number of samples ($n > 250$), the probability that the MLE is at $\theta \approx -\pi/2$ (the less favorable mode) goes to zero.
		Fig.~\ref{fig:ellipsemisspecify}(b) shows KL averaging mimics the bi-modal distribution of the global MLE
		across data samples; the less likely mode vanishes slightly slower. In contrast, the linear average takes the arithmetic average of local models from both of these two local modes, giving unreasonable parameter estimates that are close to neither (Fig.~\ref{fig:ellipsemisspecify}(c)).

\begin{figure}
	\centering
	\begin{tabular}{cccc}
	\includegraphics[width=2.5cm]{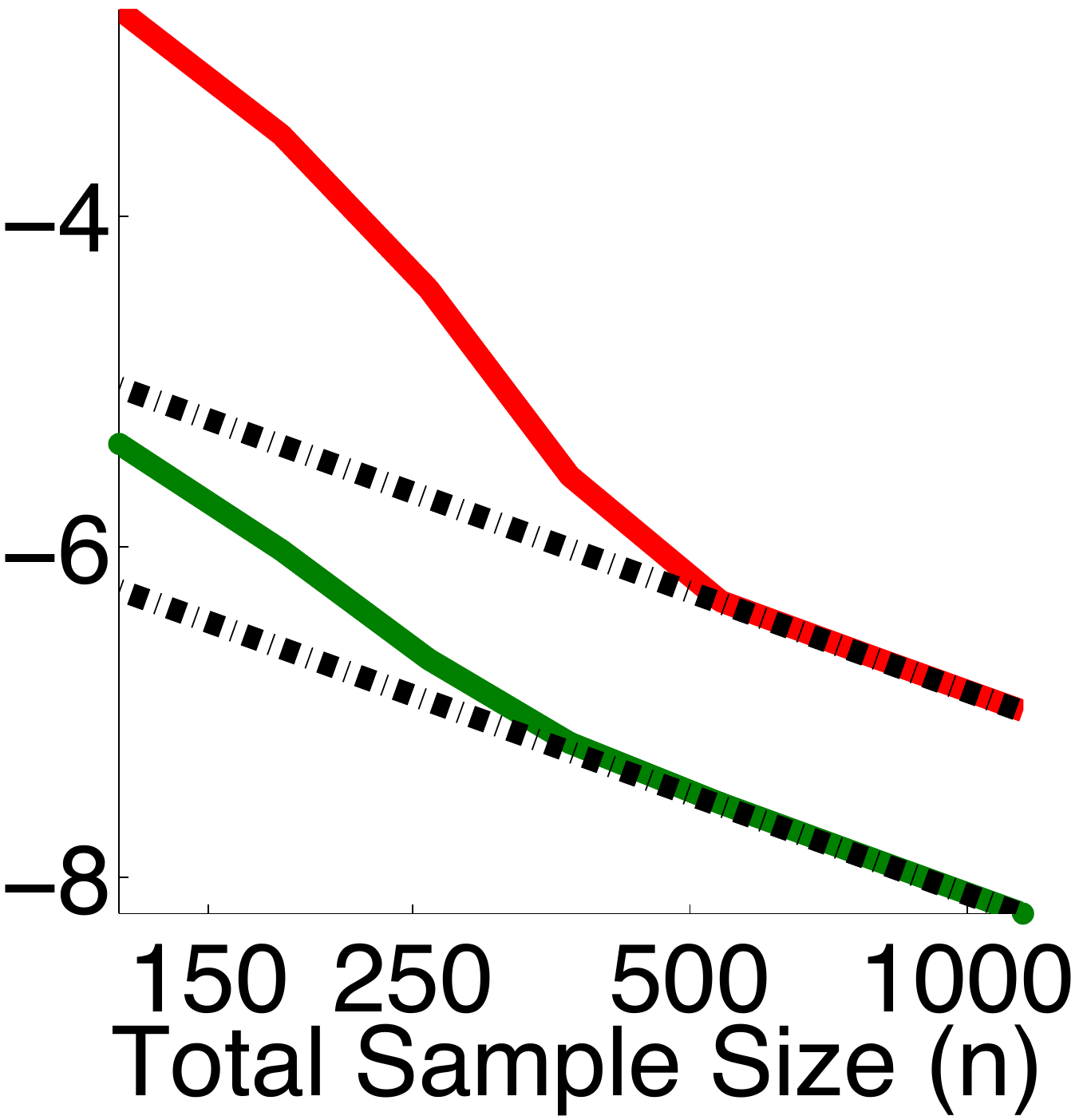}
	\hspace{-1.3cm}\raisebox{1.8cm}{\includegraphics[width=1.5cm]{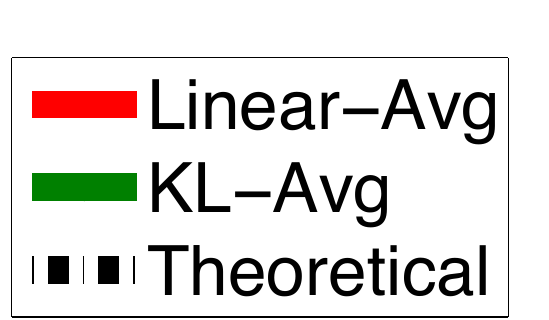}} 
	&
	\includegraphics[width=2.5cm]{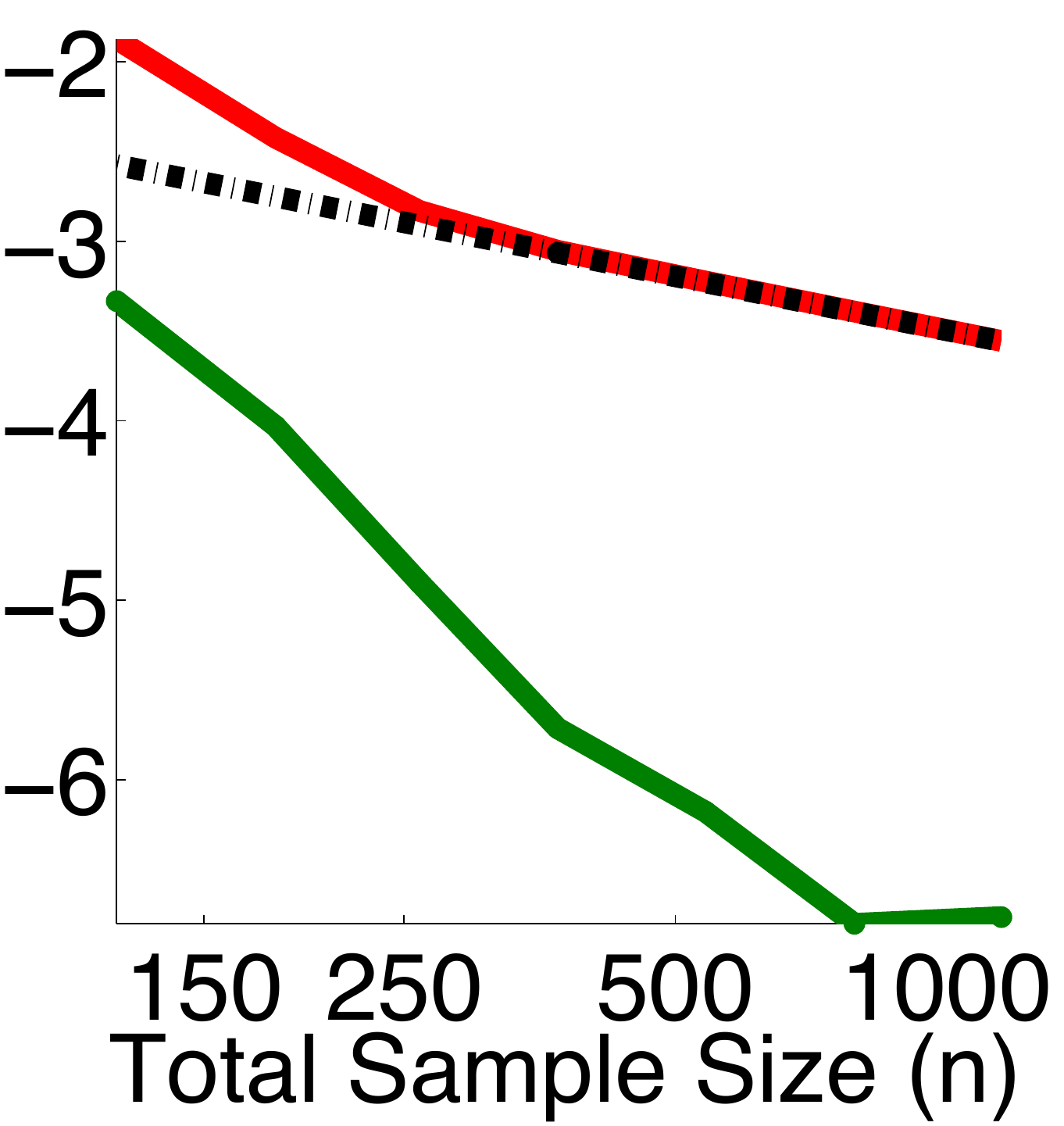}
	&
	\includegraphics[width=2.5cm]{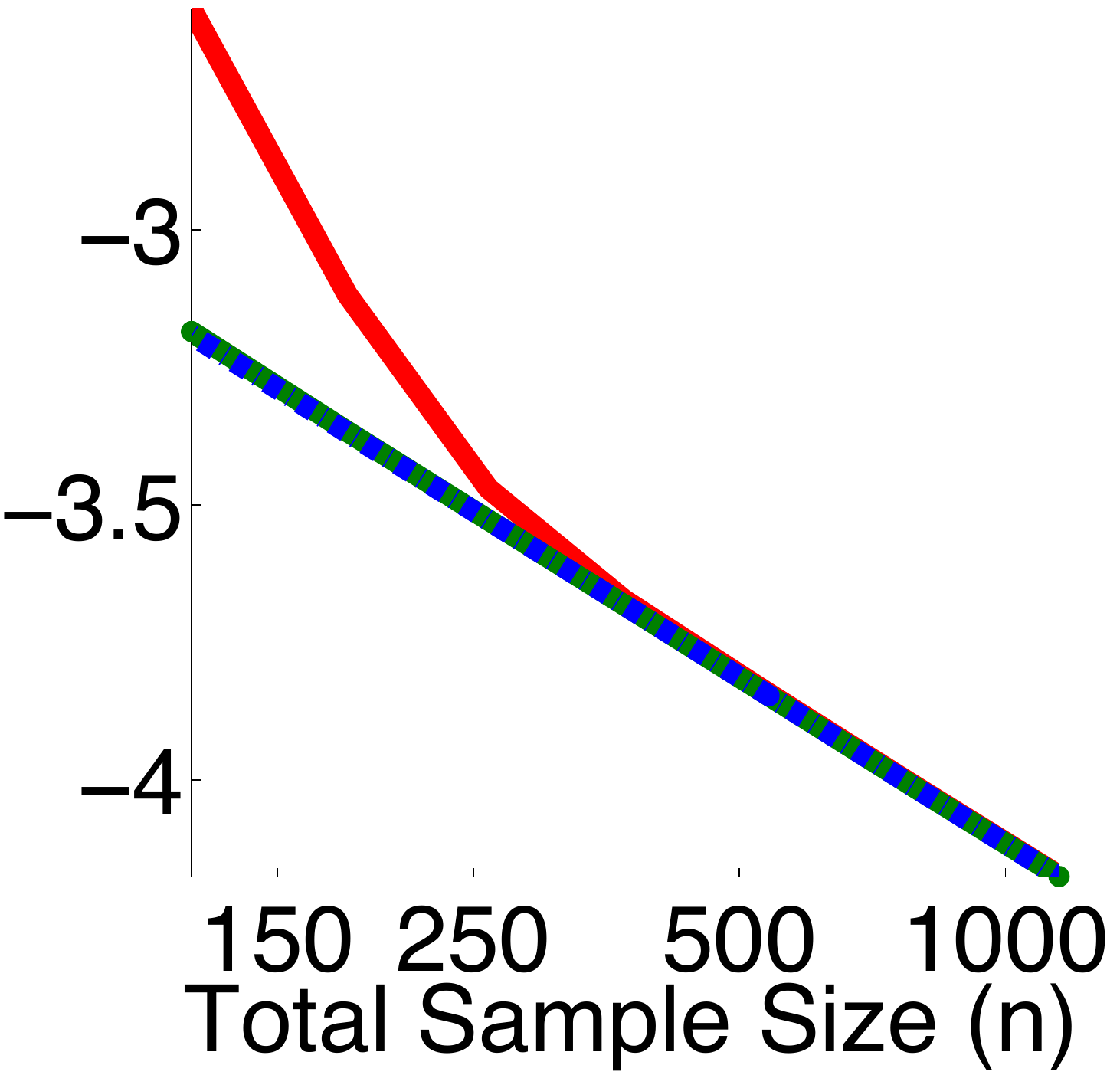}	
	\hspace{-1.3cm}\raisebox{1.8cm}{\includegraphics[width=1.5cm]{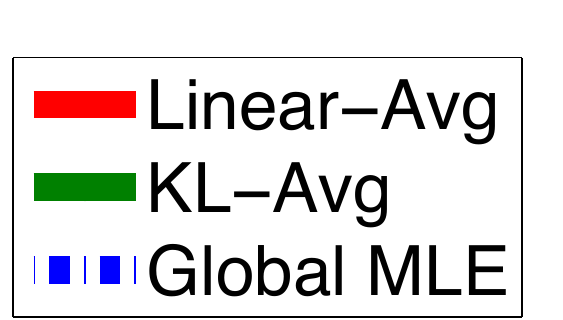}} 
	&	
	\includegraphics[width=2.5cm]{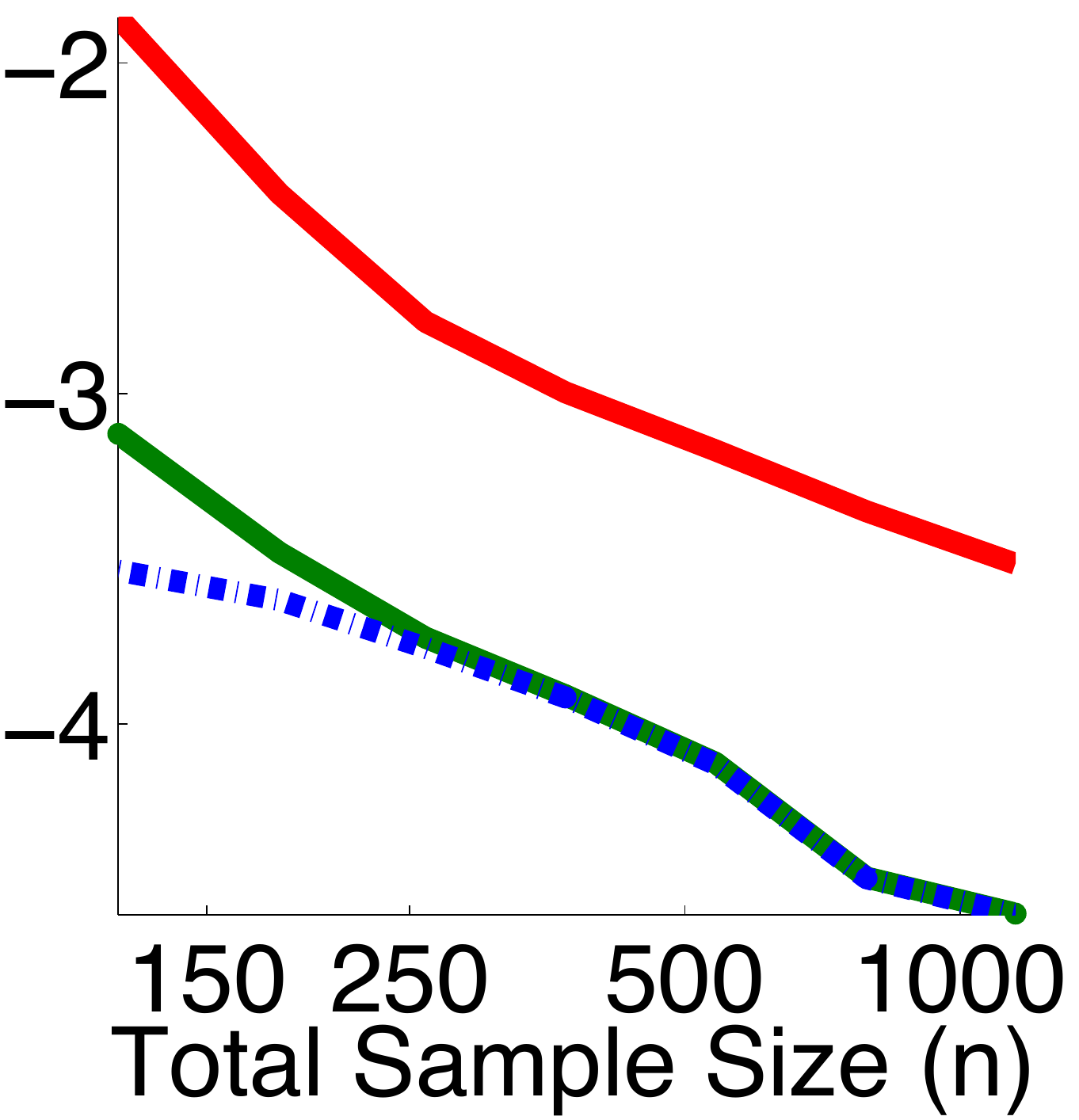}		\\
	{\small (a). $\E (|| \theta^f-\thetamle||^2)$} & {\small (b). $|\E ( \theta^f-\thetamle)|$}  &{\small (c). $\E (|| \theta^f-\thetatrue||^2)$} & {\small (d). $|\E (\theta^f-\thetatrue)|$}  
	\end{tabular}
	\caption{Result on the toy model in Example~\ref{exa:ellipseexample}. 
	 (a)-(d) The mean square errors and biases of the linear average $\thetalinear$ and the KL average $\thetakl$ w.r.t. to the global MLE $\thetamle$ and the true parameter $\thetatrue$, respectively. The y-axes are shown on logarithmic (base 10) scales.}
	\label{fig:ellipse}
\end{figure}

\begin{figure}[tbh]
	\centering
\scalebox{1}{	
	\begin{tabular}{ccc}	\hspace{0cm}\raisebox{0cm}{\includegraphics[width=.3\textwidth]{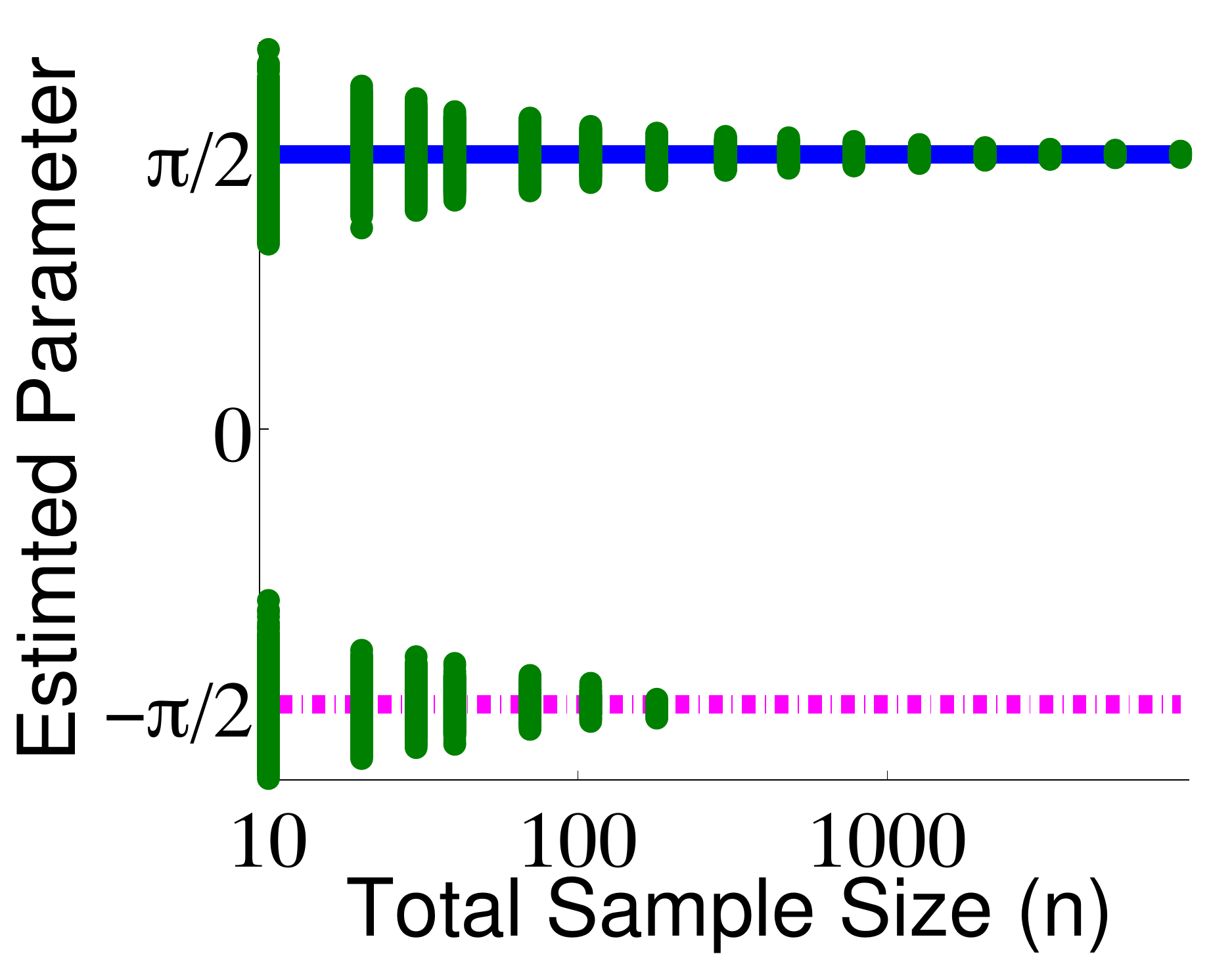}} 	
		\hspace{-1.5cm}\raisebox{.9cm}{\includegraphics[width=.1\textwidth]{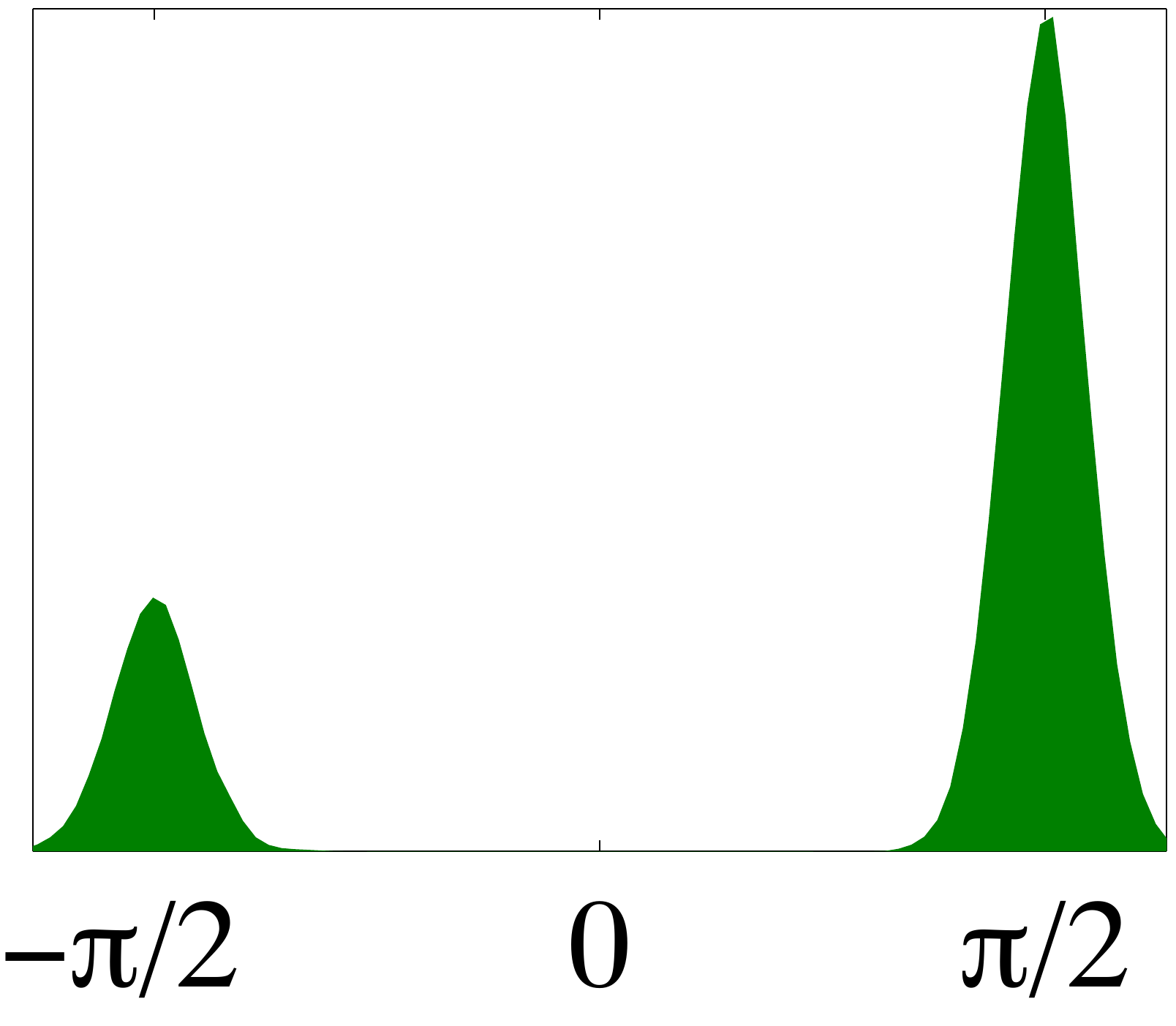}} 
		&
	\hspace{0cm}\raisebox{0cm}{\includegraphics[width=.3\textwidth]{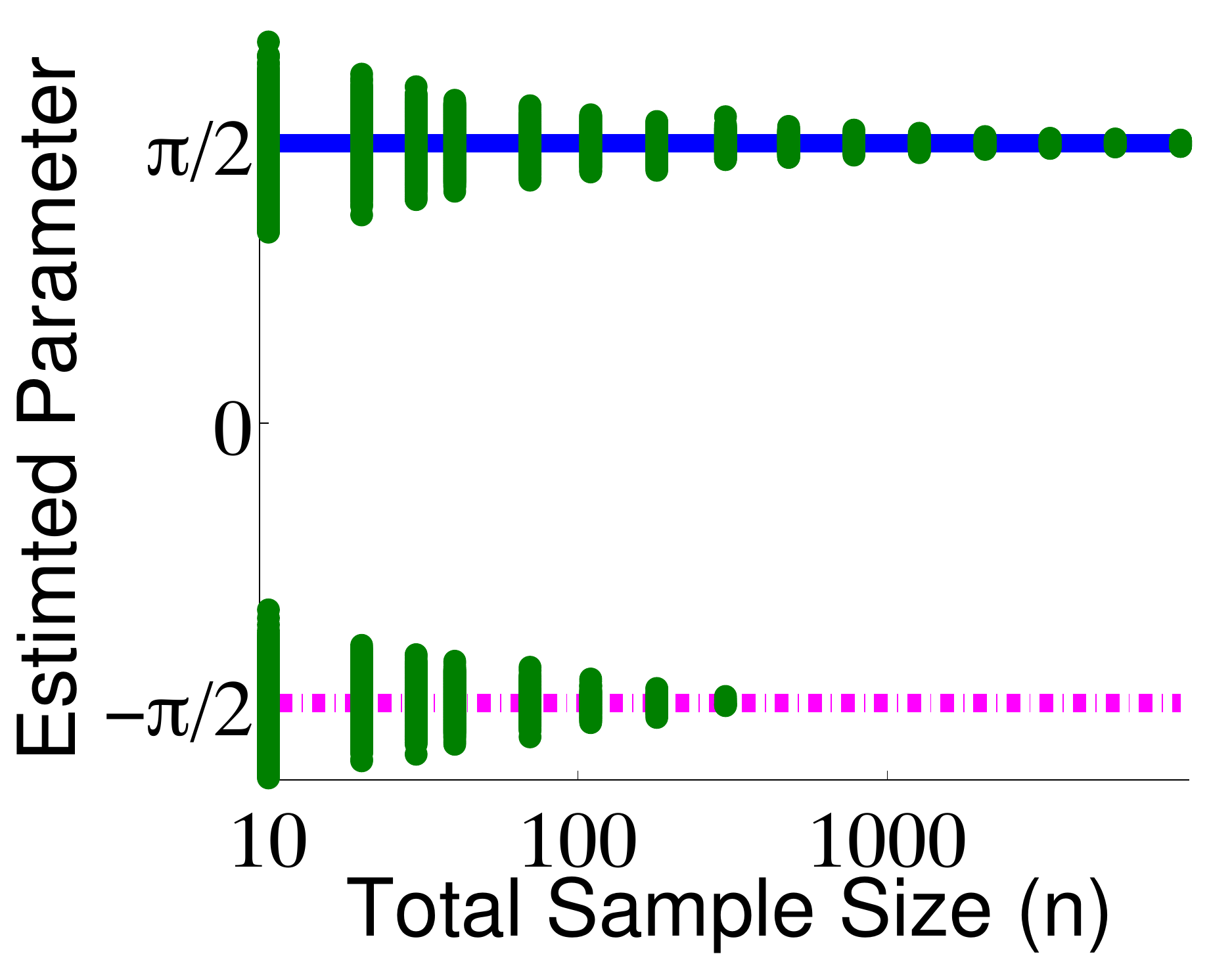}}
		\hspace{-1.5cm}\raisebox{0.9cm}{\includegraphics[width=.1\textwidth]{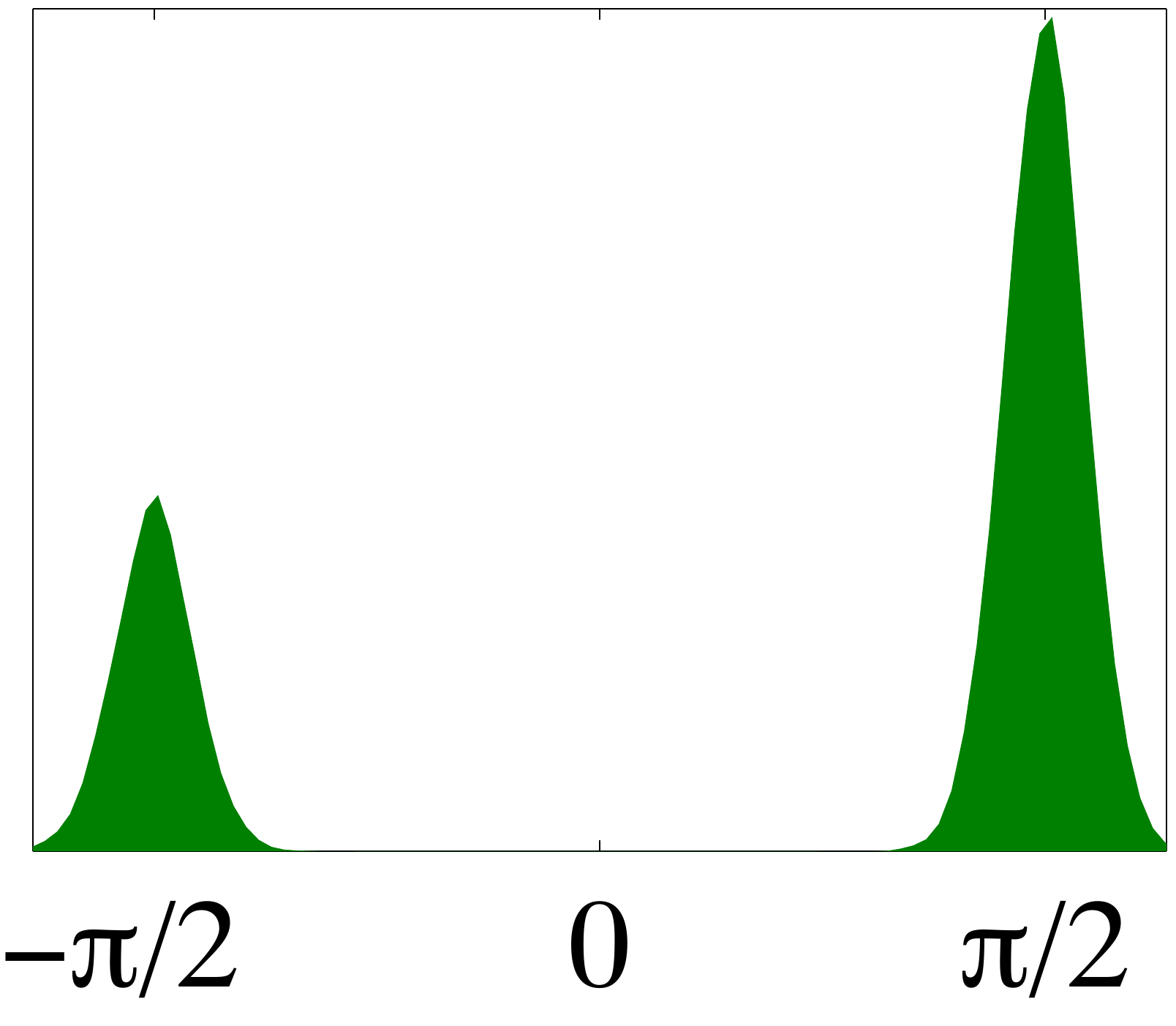}} 
		&
	\hspace{0cm}\raisebox{0cm}{\includegraphics[width=.3\textwidth]{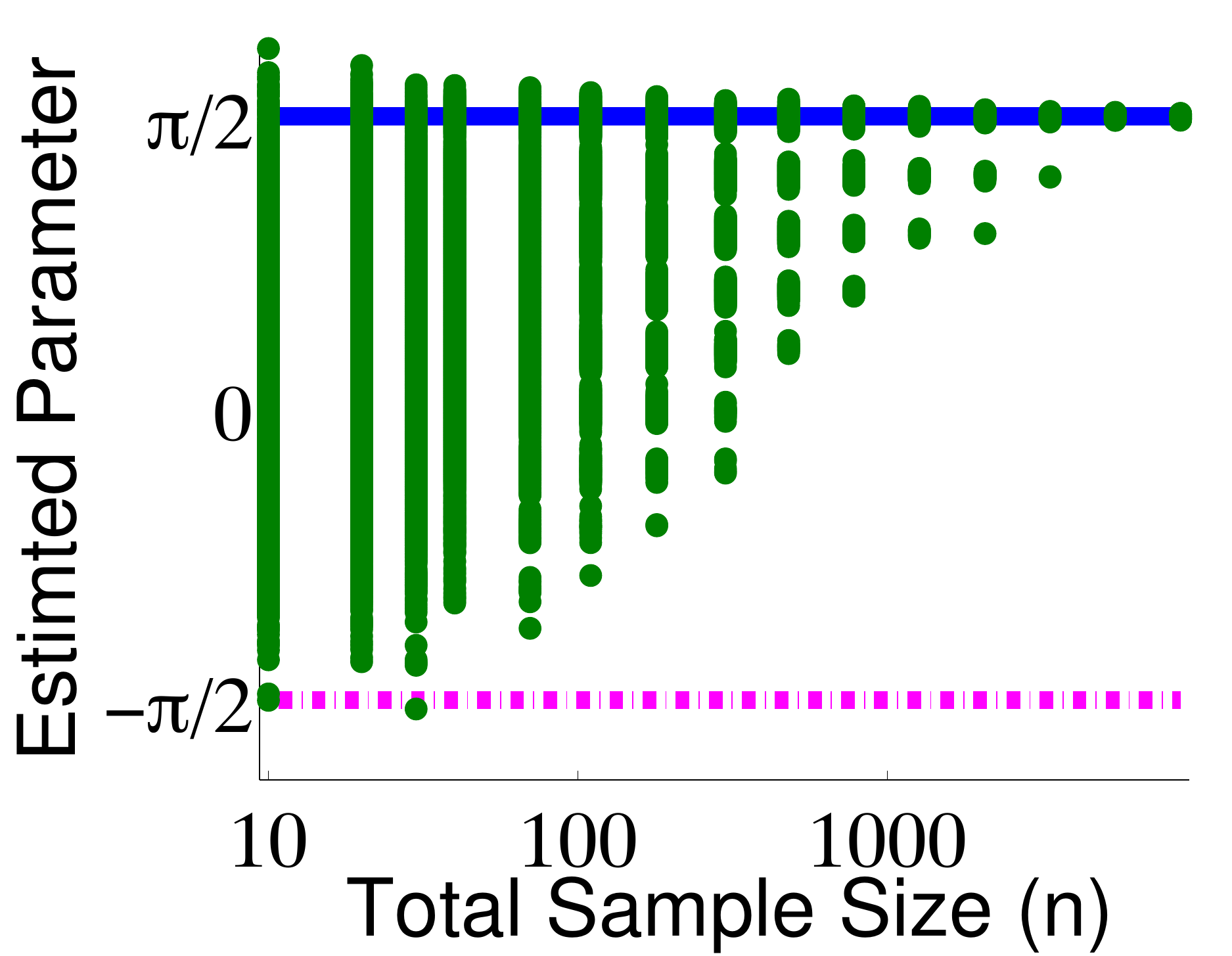}} 	
		\hspace{-1.5cm}\raisebox{.9cm}{\includegraphics[width=.1\textwidth]{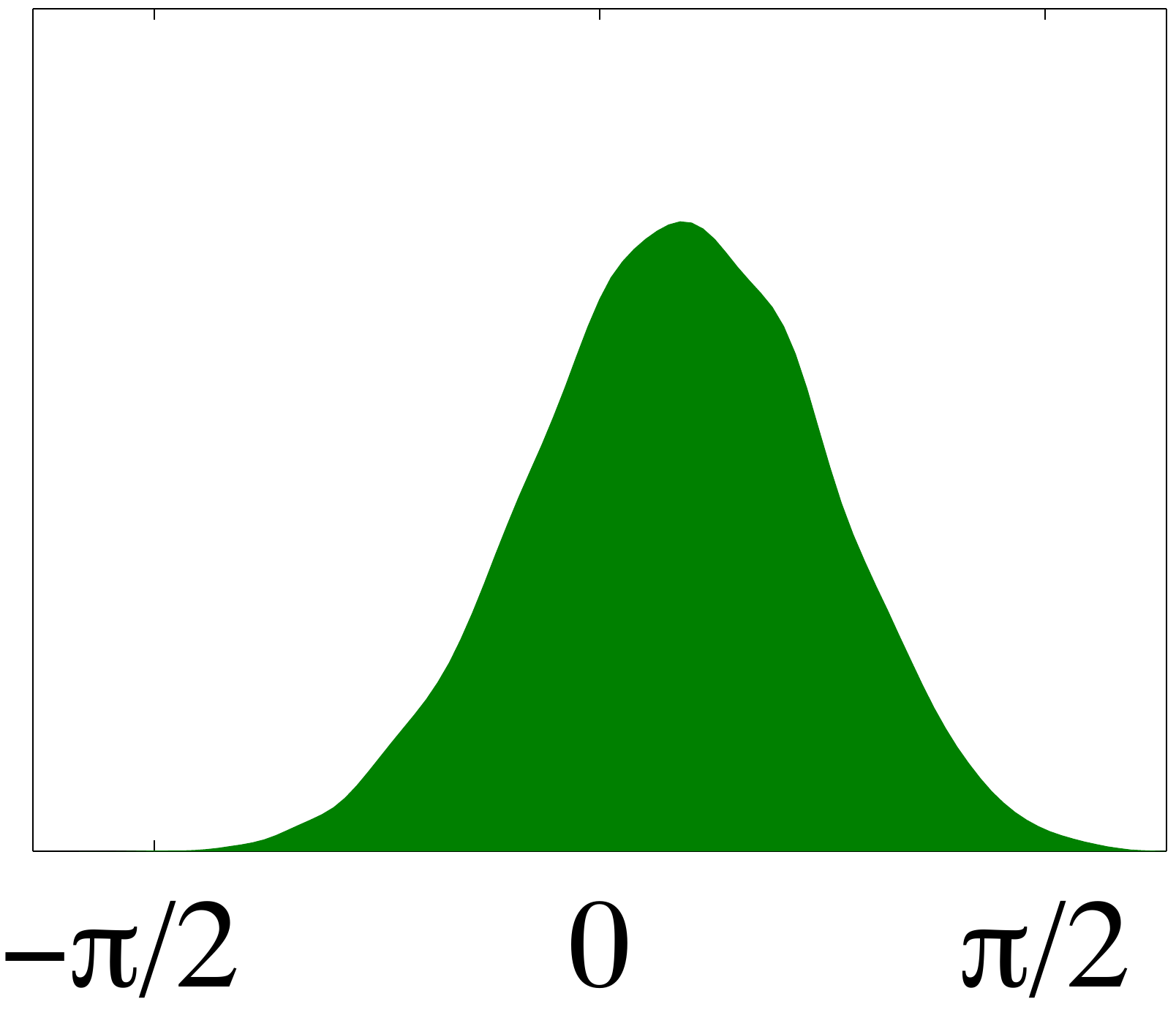}}	\\
	{\small (a). Global MLE $\thetamle$} & 	{\small (b). KL Average $\thetakl$} & 	{\small (c). Linear Average $\thetalinear$} 		
	\end{tabular}
\setlength{\unitlength}{5cm}	
\begin{picture}(0,0) 
\put(-2.18, 0.12){{\scriptsize $(n=10)$}}
\put(-1.26, 0.12){{\scriptsize $(n=10)$}}
\put(-.34, 0.12){{\scriptsize $(n=10)$}}
 \end{picture}
 }	
	\caption{Result on the toy model in Example~\ref{exa:ellipseexample} 
	with model misspecification: scatter plots of the estimated parameters vs. the total sample size $n$ (with 10,000 random trials for each fixed $n$). The inside figures are the densities of the estimated parameters with fixed $n=10$. Both global MLE and KL-average concentrate on two locations $(\pm\pi/2)$, and the less favorable $(-\pi/2)$ vanishes when the sample sizes are large (e.g., $n>250$). In contrast, the linear approach averages local MLEs from the two modes, giving
unreasonable estimates spread across the full interval.}%
	\label{fig:ellipsemisspecify}
\end{figure}

\subsection{Gaussian Mixture Models on Real Datasets}
We next consider learning Gaussian mixture models.
Because component indexes may be arbitrarily switched,
na{\"i}ve linear averaging is problematic; %
we consider a \emph{matched linear average} that first matches 
indices %
by minimizing the sum of the symmetric KL divergences of the different mixture components.  
The KL average is also difficult to calculate exactly, since the KL divergence between Gaussian 
mixtures 
is intractable. We approximate the KL average using Monte Carlo sampling 
(with 500 samples per local model), 
corresponding to the parametric bootstrap 
discussed in Section~\ref{sec:setting}.

We experiment on the MNIST dataset and the YearPredictionMSD dataset in the UCI repository,
where the training data is partitioned into 10 sub-groups randomly and evenly.  
In both cases, we use the original training/test split; we use the full testing set, and vary the number of training examples $n$ by randomly sub-sampling from the full training set (averaging over 100 trials). 
We take the first 100 principal components when using MNIST. 
Fig.~\ref{fig:mnist}(a)-(b) and \ref{fig:yearprediction}(a)-(b) show the training and test likelihoods. 
As a baseline, we also show the average of the log-likelihoods of the local models (marked as {\tt local MLEs} in the figures); this corresponds to randomly selecting a local model as the combined model. 
We see that the KL average tends to perform as well as the global MLE, and remains stable even with small sample sizes. 
The na{\"i}ve linear average performs badly even with large sample sizes. The matched linear average performs as badly as the na{\"i}ve linear average when the sample size is small, but improves towards to the global MLE as sample size increases. 

For MNIST, we also consider a severely heterogenous data partition by splitting the images into $10$ groups according to their digit labels. %
In this setup, each partition learns a local model only over its own digit, with no information about the other digits. 
Fig.~\ref{fig:mnist}(c)-(d) shows the KL average still performs as well as the global MLE, but both the na{\"i}ve and matched linear average are much worse even with large sample sizes, due to the dissimilarity in the local models.

\begin{figure}
	\centering
	\begin{tabular}{cccc}
	\hspace{-.05cm}\raisebox{0cm}{\includegraphics[width=.22\textwidth]{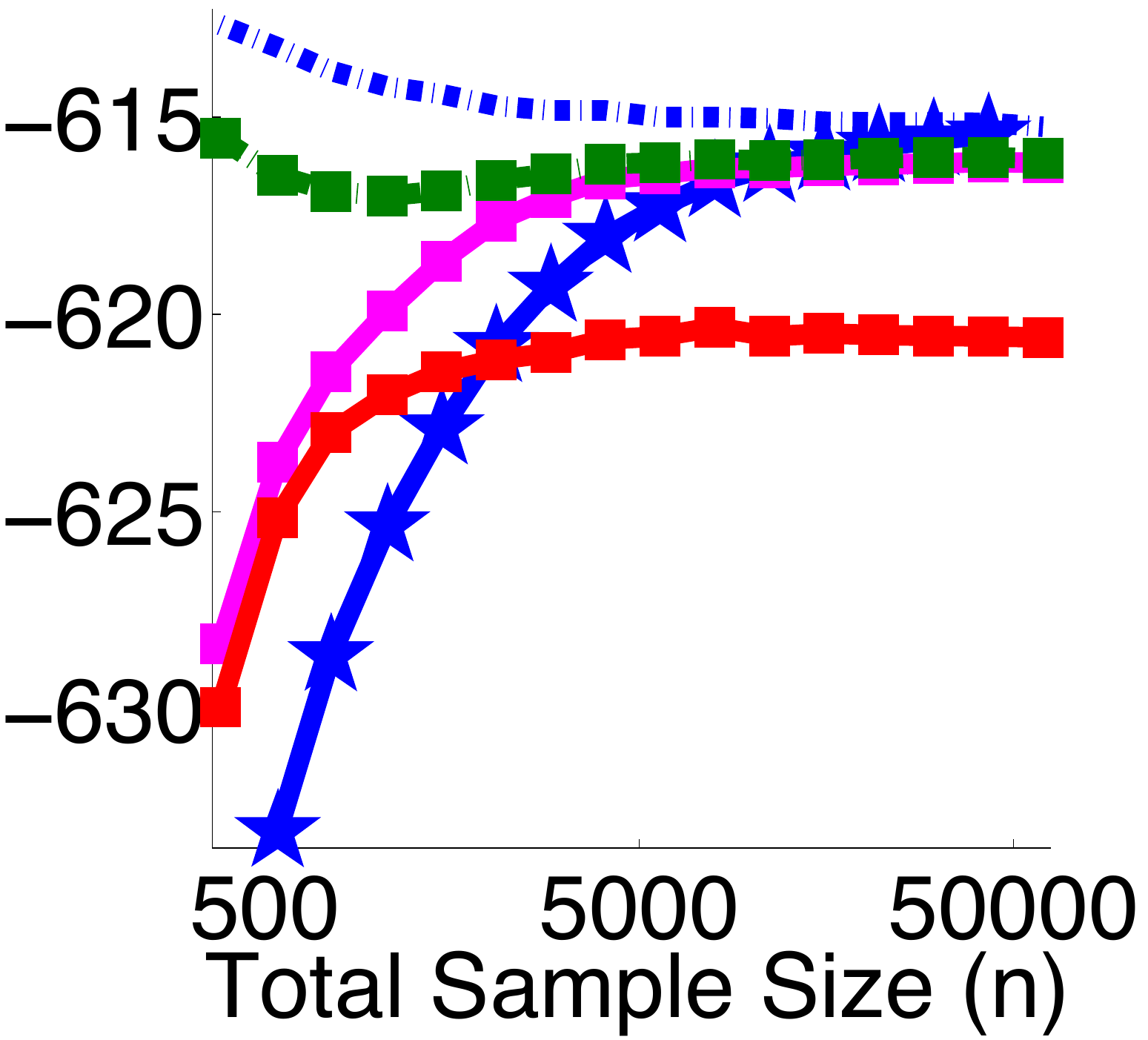}} 	&
	\hspace{-.15cm}\raisebox{0cm}{\includegraphics[width=.22\textwidth]{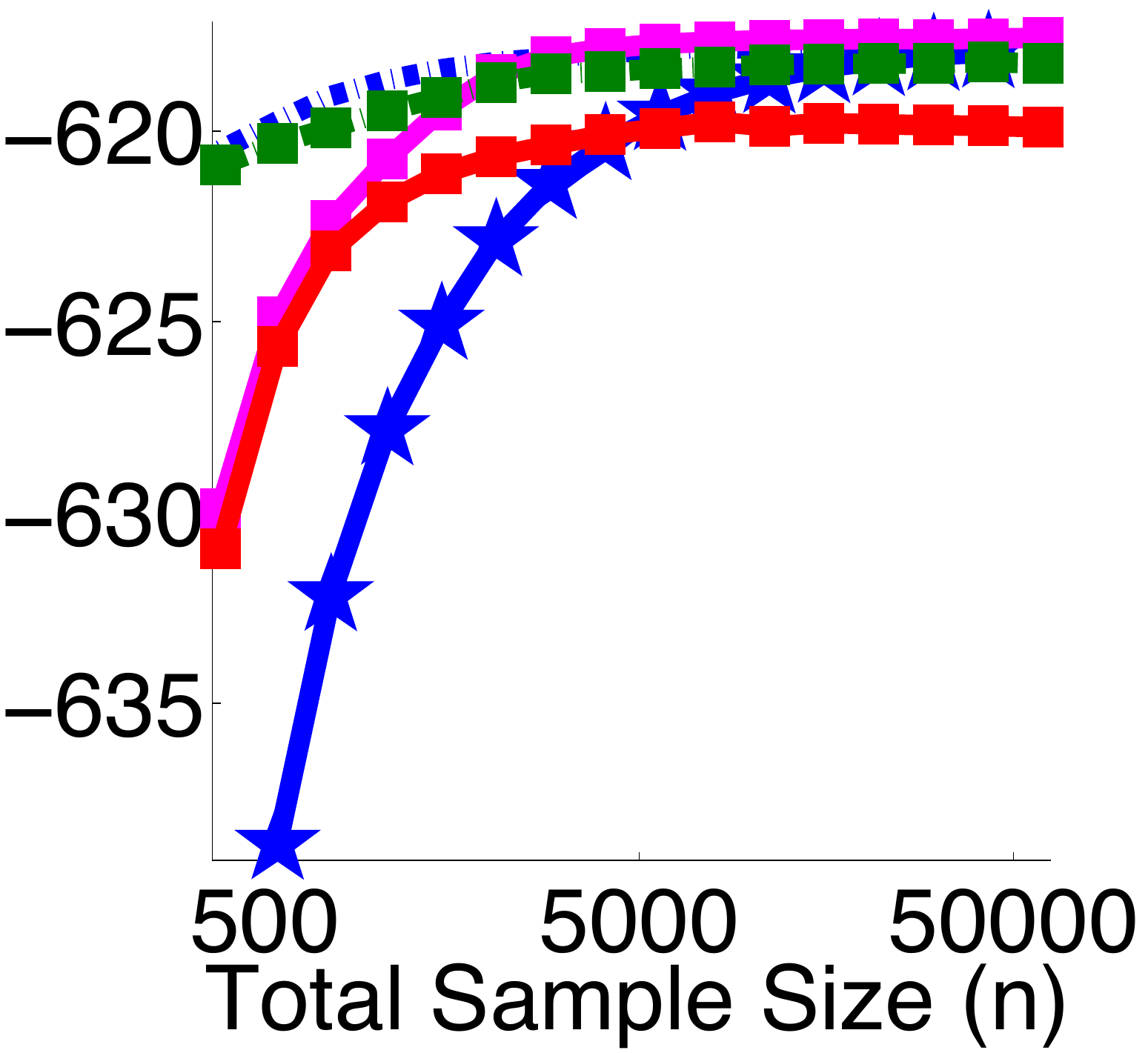}} 
	\hspace{-1.8cm}\raisebox{.6cm}{\includegraphics[width=.12\textwidth]{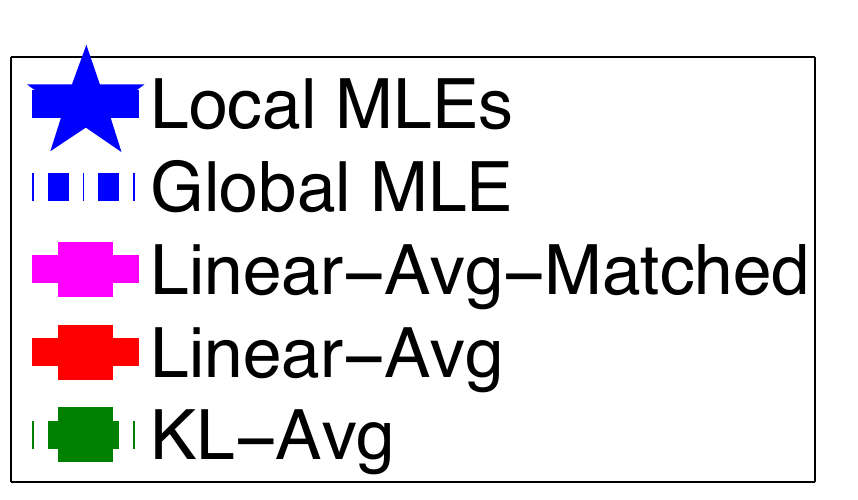}} 	
	\ & \ 
	\hspace{-.0cm}\raisebox{0cm}{\includegraphics[width=.22\textwidth]{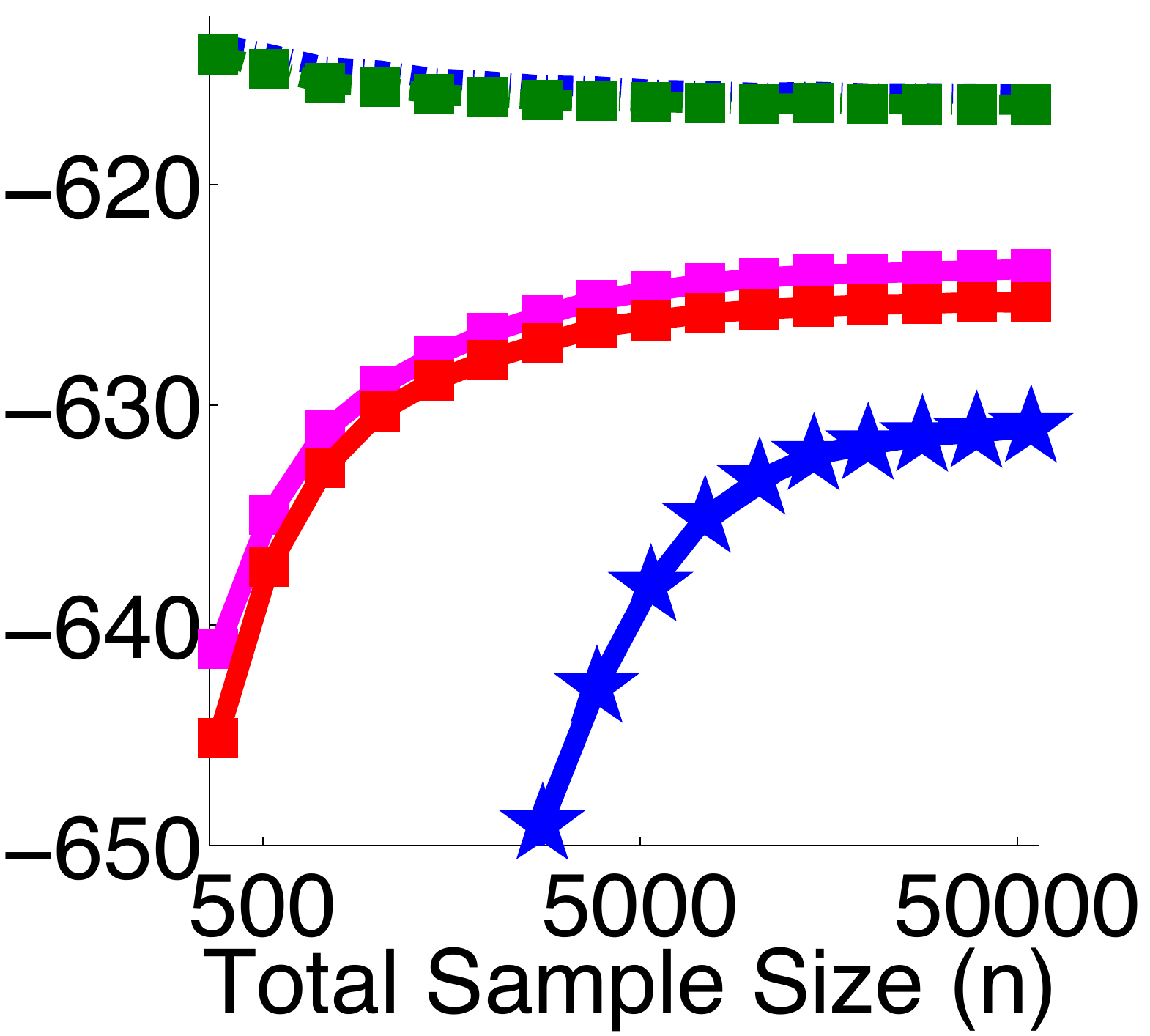}} 	&
	\hspace{-.2cm}\raisebox{0cm}{\includegraphics[width=.22\textwidth]{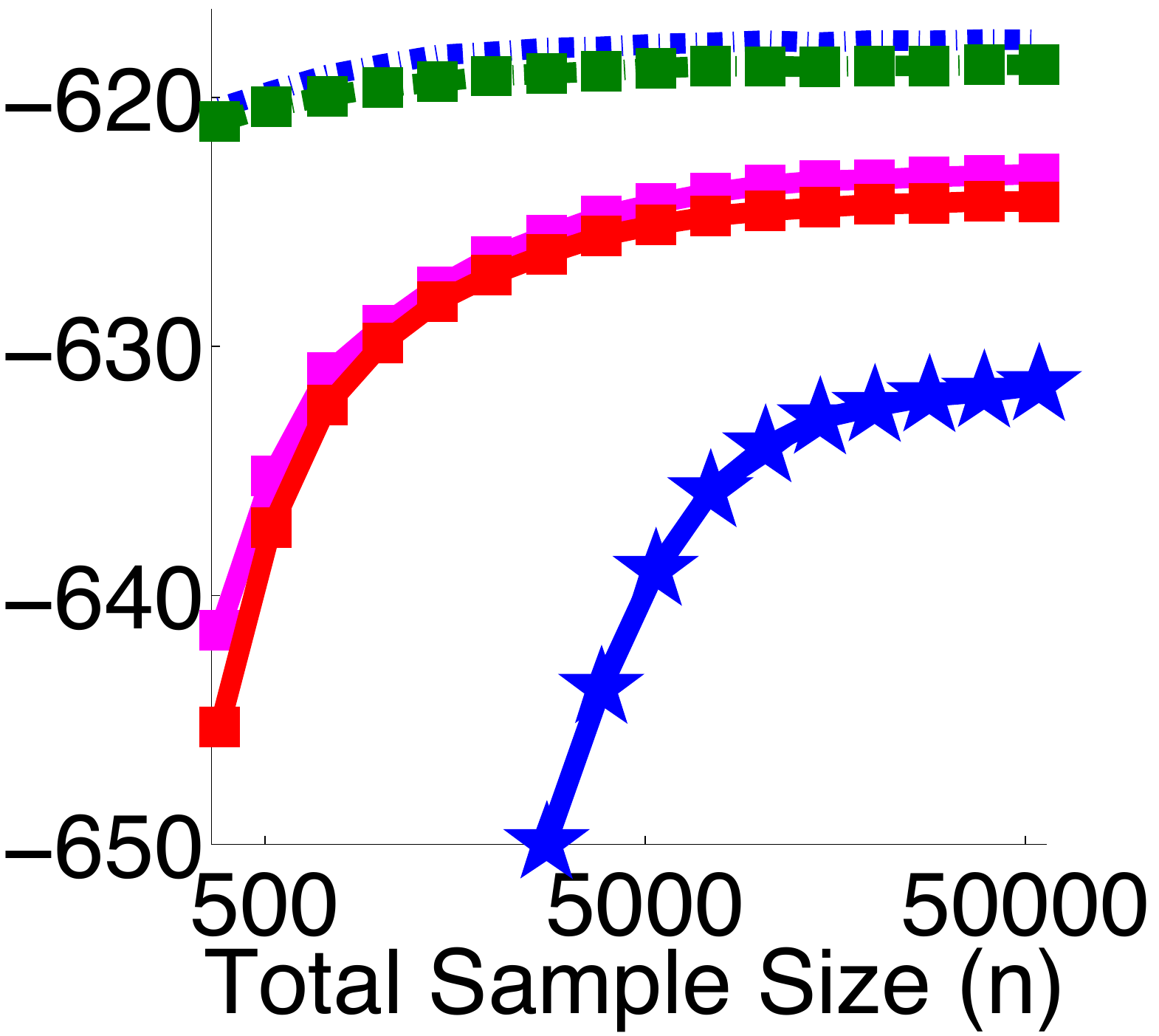}}
	\\[-.5em]
	 \makecell{\small (a) Training LL \\[-.25em] \small ~~(\emph{random partition})} & 
	\makecell{\small (b) Test LL \\[-.25em]      \small ~~~~(\emph{random partition})} & 
	\makecell{\small (c) Training LL \\[-.25em]  \small ~~~~~~~(\emph{label-wise partition})} & 
	\makecell{\small (d) Test LL \\[-.25em]      \small ~~~~~~~(\emph{label-wise partition})} 
	\end{tabular}\\[-.9em]
	\caption{Learning Gaussian mixture models on MNIST: training and test log-likelihoods of different methods with varying training size $n$.  In (a)-(b), the data are partitioned into 10 sub-groups uniformly at random (ensuring sub-samples are i.i.d.); in (c)-(d) the data are partitioned according to their digit labels. The number of mixture components is fixed to be 10. 
}
	\label{fig:mnist}
\end{figure}

\begin{figure} 
  \begin{minipage}[c]{0.5\textwidth}
  \begin{tabular}{cc}
	\hspace{0cm}\raisebox{0cm}{\includegraphics[width=.44\textwidth]{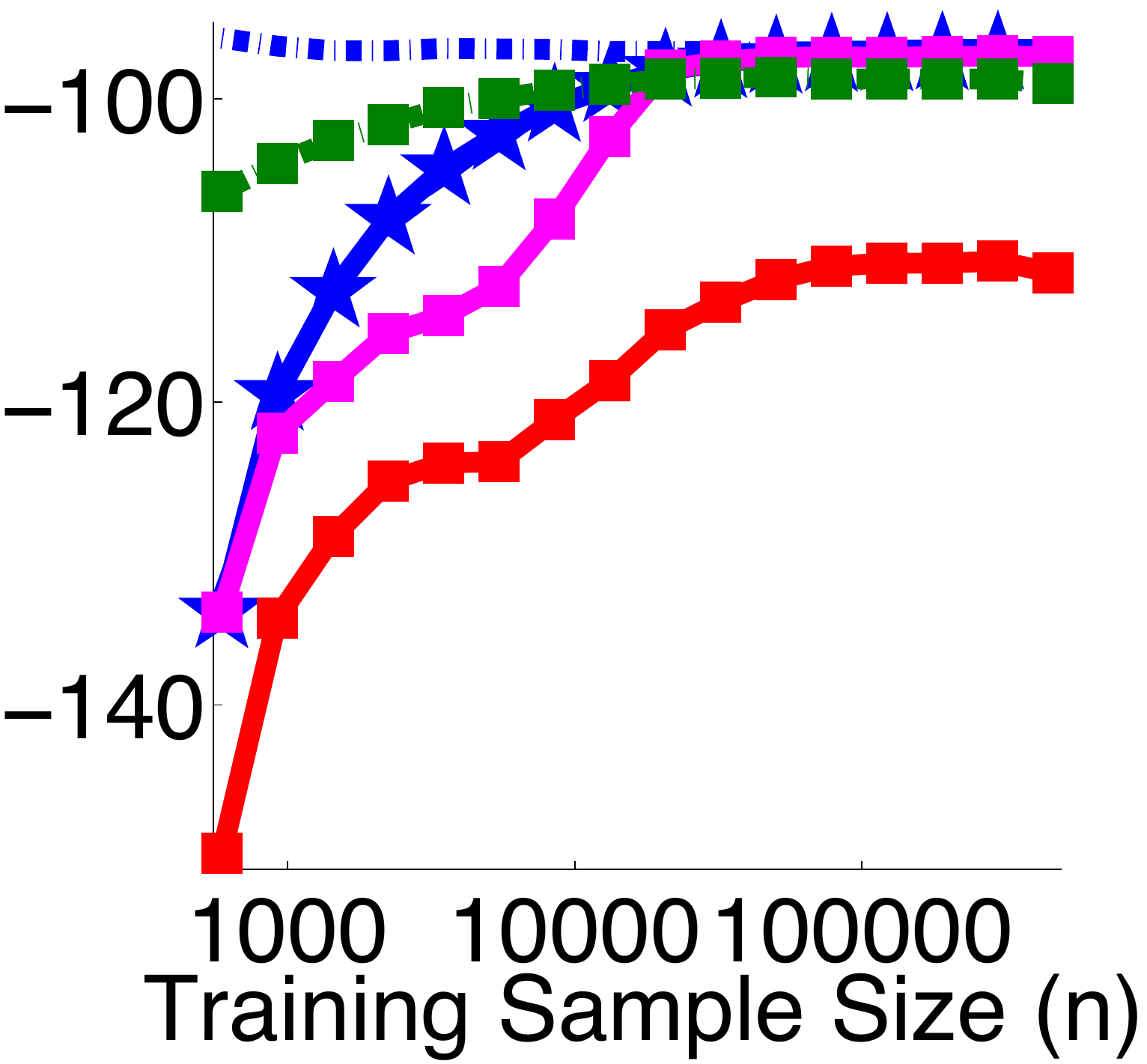}} 
	\hspace{-1.6cm}\raisebox{.5cm}{\includegraphics[width=.24\textwidth]{figures/Train_mnist_unifpartition_v1_legend}}%
			&
	\hspace{-0cm}\raisebox{0cm}{\includegraphics[width=.44\textwidth]{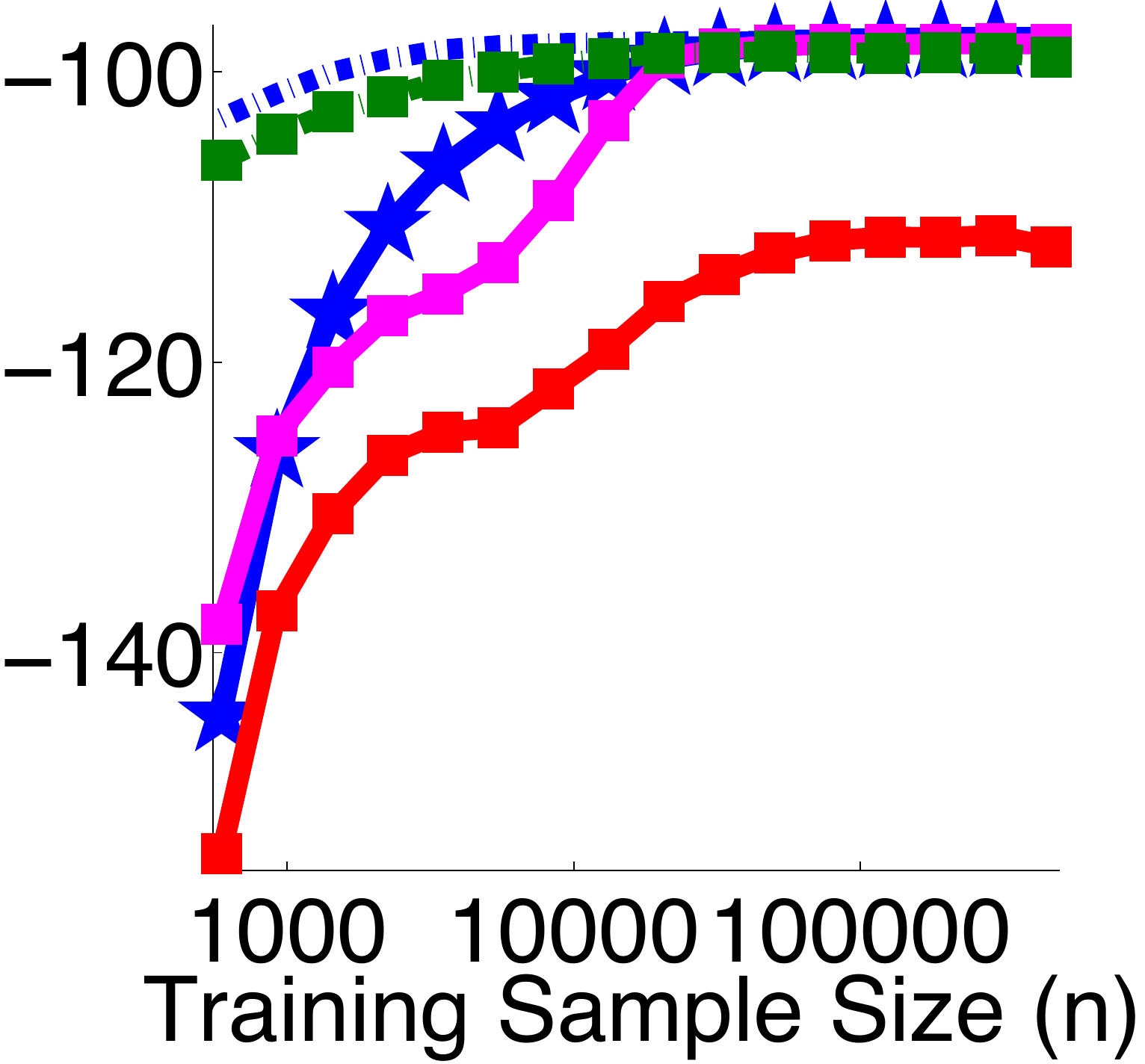}} 		
	\\[-.25em]
	 {\small (a) Training log-likelihood} & 
	{\small (b) Test log-likelihood} 	
	\end{tabular}
  \end{minipage}%
  \quad
 \raisebox{.5cm}{      
  \begin{minipage}[c]{0.45\textwidth}
    \caption{Learning Gaussian mixture models on the YearPredictionMSD data set. 
 The data are randomly partitioned into 10 sub-groups, and we use 10 mixture components. 
}
	\label{fig:yearprediction}
  \end{minipage}
  }\\[-.9em]
\end{figure}

\vspace{-.25\baselineskip}
\section{Conclusion and Future Directions}
\vspace{-.25\baselineskip}
We study communication-efficient algorithms for learning generative models with distributed data. 
Analyzing both a common linear averaging technique and a less common KL-averaging technique provides both  
theoretical and empirical insights. Our analysis opens many important future directions, including 
extensions to high dimensional inference and efficient approximations for complex machine learning models, 
such as LDA and neural networks. %

\vspace{-.5\baselineskip}
\paragraph{Acknowledgement.} This work supported in part by NSF grants IIS-1065618 and IIS-1254071.

\clearpage \newpage

\include{for_include_appendix}

\bibliographystyle{unsrtnat}
\bibliography{bib_nips14_distributed_KL_clean}

\appendix 
This document contains proofs and other supplemental information for the NIPS 2014 submission, ``Distributed Estimation, Information Loss and Curved Exponential Families".  

\section{Curved Exponential Families}

\paragraph{Notation.}
\label{(Differential).}
Denote by $\dimtheta$ the dimension of $\theta$, and $\dimeta$ the dimension of $\eta$ and $\phi(x)$. 
We use the following notations for derivatives, 
\begin{align*}
\dot\eta^i_j(\theta) = \myp{\eta^i(\theta)}{\theta_j} ~~~~ \text{and} ~~~~ \ddot\eta^i_{jk}(\theta) = \myp{^2\eta^i(\theta)}{\theta_{j}\partial\theta_k}. 
 \end{align*}
We write $\dot\eta_{\theta}  = [\dot\eta^i_j(\theta)]_{ij}$and $ \ddot\eta_i(\theta) = [\ddot\eta^j_{ik}(\theta) ]_{jk}$,  both of which are ($\dimeta\times\dimtheta$) matrices. 
\label{(Expectation).}
Denote by $\E_{\theta}$ the expectation under $p(x | \theta)$, and $\E_{X}$ the empirical average under sample $X = \{ x^i\}_{i=1}^n$, e.g., $\E_{X}\log p(x | \theta)  = \frac{1}{n} \sum_i \log p(x^i | \theta)$. 
\label{(Fisher Information).} 
Denote by $I_{\theta}$ the $(\dimtheta \times \dimtheta)$ Fisher information matrix of $p(x | \theta)$, that is, $I_{\theta} = - \E_\theta ({\partial^2 \log p(x | \theta)}/{\partial^2\theta})$. 
Define $\Sigma_{\theta} = \cov_{\theta}(\phi)$, the $(\dimeta \times \dimeta)$ Fisher information matrix of the full exponential family $p(x  | \eta) = \exp(\eta^T \phi(x) - \log Z(\eta))$; one can show that $I_{\theta} = \dot\eta_{\theta}^T \Sigma_{\theta} \dot\eta_{\theta}$.
We use $\identitymatrix{m}$ to denote the ($m\times m$ ) identity matrix to distinguish from the Fisher information $I$. 
\label{(Moments).}
We denote by $\moment(\theta) = \E_{\theta}[\phi(x)] $ the mean parameter. Note that $\moment(\theta)$ is a differentiable function of $\theta$, with $\dot\moment(\theta) = \Sigma_\theta \dot\eta_{\theta}$. For brevity, we write $\mu_{\theta} = \mu(\theta)$, and $\momentmle = \moment_{\thetamle}$, and $\momentX = \E_X[\phi(x)] = \frac{1}{n}\sum_i \phi(x^i)$. 
\label{(True  Para).}
We always denote the true parameter by $\thetatrue$, and use the subscript (or superscript) $``*"$ to denote the cases when $\theta = \thetatrue$, e.g.,  $\momenttrue = \moment_{\thetatrue}$, and
$\Sigma_* = \Sigma_{\thetatrue}$. %

\label{$\Op$ notation}
We use the big O in probability notation. For a set of random variables $X_n$ and constants $r_n$, the notation $X_n = \op(r_n)$ means that $X_n / r_n$ converges to zero in probability, and correspondingly, the notation $X_n = \Op(r_n)$ means that $X_n / r_n$ is bounded in probability. For our purpose, note that $X_n = \Op(r_n)$ (or $\op(r_n)$) implies that $\E(X_n^{\alpha}) = \OO(r_n^{\alpha})$ (or $\oo(r_n^{\alpha})$), $\alpha = 1, 2$. 
We assume the MLE exists and is strongly consistent, and will ignore most of the technical conditions of asymptotic convergence in our proof; see e.g., \citet{barndorff1978information, ghosh1994higher, kass2011geometrical} for complete treatments.  

We start with some basic properties of the curved exponential families. %

\begin{lem}
\label{lem:derivatives}
For Curved exponential family $p(x | \theta)  = \exp(\eta(\theta)^T \phi(x) - \Phi(\eta(\theta)))$, we have
\begin{align*}
\myp{\log p(x  | \theta)}{\theta}   &  =  \dot{\eta}(\theta)^T (\phi(x) - \E_{\theta}(\phi(x))),  \\
\myp{^2\log p(x | \theta)}{\theta_i\theta_j}  
 &=  \ddot{\eta}_{ij}(\theta)^T [ \phi(x) - \E_{\theta}(\phi(x)) ]  - \dot{\eta}_i(\theta)^T ~\Sigma_{\theta}  ~ \dot{\eta}_j(\theta),  
\end{align*}
where $\ddot\eta_{jk}(\theta) = [\ddot\eta^i_{jk}(\theta)]_i$ is a $\dimeta \times 1$ vector. 
\end{lem}
The following higher order asymptotic properties of MLE play an important role in our proof. 
\begin{lem}[First and Second Order Asymptotics of MLE]
\label{lem:higherMLE}
Denote by $\hat s_X =  \E_X[\dot \ell(\thetatrue; x)]  = \dot\eta_*^T ( \momentX - \momenttrue )$. We have the first and second order expansions of MLE, respectively, 
\begin{align*}
&I_*(\thetamle - \thetatrue) =  \hat s_X + \Op(n^{-1}),  \\
&[I_*(\thetamle - \thetatrue) - \hat s_X]_i   = 
(\momentX - \momenttrue)^T L_i (\momentX - \momenttrue)   + \Op(n^{-3/2}). 
\end{align*}
where 
\begin{align*}
L_i = \frac{1}{2}(\ddot\eta_i^* I_*^{-1} \dot\eta_*^T + \dot\eta_* I_*^{-1} (\ddot\eta_i^*)^T )  +   \frac{1}{2}  \dot\eta_*  I_*^{-1} J_i I_*^{-1} \dot\eta_*^T,  
\end{align*}
and $J_i$ is a $\dimtheta\times\dimtheta$ matrix whose $(m, l)$-element is 
$ \E_{\thetatrue} \left[  \myp{\log p(x |\thetatrue )}{\theta_i\partial\theta_m \partial \theta_l }  \right] $.  %
\end{lem}
\begin{proof}
See \citet{ghosh1994higher}.   
\end{proof}

\renewcommand\windowpagestuff{\leavevmode\put(0,-50){\includegraphics[width=3cm]{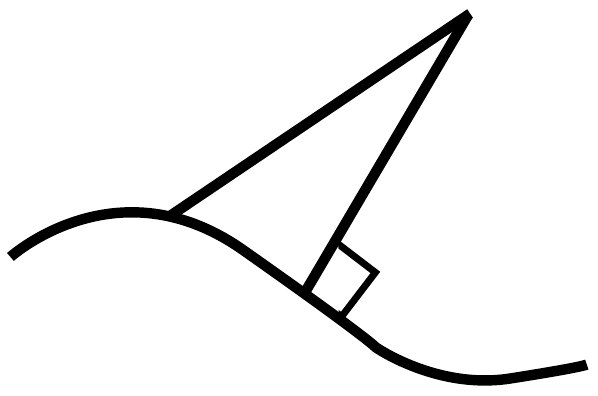}}}
\opencutright

\begin{lem}
\label{lem:projection}
Let $\momentX= \E_X [\phi(x)] = \frac{1}{n}\sum_{i=1}^n \phi(x^i)$ be the empirical mean parameter on an i.i.d. sample $X$ of size $n$, and $\momentmle=\E_{\thetamle}[\phi(x)]$ the mean parameter corresponding to the maximum likelihood estimator $\thetamle$ on $X$, we have 
\begin{align*}
\momentmle - \momenttrue   = P_* (\momentX -\momenttrue) + \Op(n^{-1}), 
&&\text{where}&&
P_* = \Sigma_{*} \dot\eta_* ( \dot\eta_*^T \Sigma_* \dot \eta)^{-1}  \dot \eta_*^T , 
\end{align*}
where the $(\dimeta\times \dimeta)$ matrix $P_{*}$ can be treated as the projection operator onto the tangent space of $\mu(\theta)$ at $\thetatrue$ in the $\mu-$space (w.r.t. an inner product defined by $\Sigma_*^{-1}$).  
See \citet{efron1975defining} for more illustration on the geometric intuition. 

Further, let $N_* = \identitymatrix{\dimeta} - P_*$, then $N_*$ 
calculates the component of a vector that is normal to the tangent space; it is easy to verify from the definition that 
\begin{align*} 
N_*^T = \Sigma_*^{-1} N_* \Sigma_* ,&&
N_*^T  \dot \eta_*  = N_*\Sigma_*\dot\eta_* = 0. 
\end{align*}
\end{lem}
\begin{proof}
The first order asymptotic expansion of $\thetamle$ show that
$$
\thetamle - \thetatrue = I_{*}^{-1}    \dot\eta_*^T (\momentX - \momenttrue)  + \Op(n^{-1}), 
$$
where $I_{*} = \dot\eta_* ^T \Sigma_{*} \dot\eta_*$ is the Fisher information. 
Using Taylor expansion, we have
\begin{align*}
\momentmle - \momenttrue
& = \moment(\thetamle) - \moment(\thetatrue) \\
& =  \dotmomenttrue[ \thetamle - \thetatrue ]  + \Op(n^{-1})\\
& =    \Sigma_{*} \dot \eta_*  [ \thetamle - \thetatrue ] + \Op(n^{-1})\\ 
& =    \Sigma_{*} \dot \eta_* [I_{*}^{-1}    \dot\eta_*^T (\momentX - \momenttrue)] + \Op(n^{-1}) . 
\end{align*}
This finishes the proof. 
\end{proof}

\newcommand{\dovern}{n^{-1}}

\begin{lem}
\label{lem:k2mle}
Consider a sample $X$ of size $n$, evenly partitioned into $d$ subgroups $\{X^k \colon k \in [d]\}$. Let $\thetamle$ be the global MLE on the whole sample $X$ and $\thetak$ be the local MLE based on $X^k$. We have 
$$
\thetamle - \thetak =  (\dot\eta_*^T \Sigma_* \dot \eta_*)^{-1} \dot\eta_*^T (\momentX - \momentXk) + \Op(\dovern),
$$
where $\momentX$ and $\momentXk$are the empirical mean parameter on the whole sample $X$ and the subsample $X^k$, respectively; note that $\momentX$ is the mean of $\{\momentXk\}$, that is, $\momentX = d^{-1}\sum_k \momentXk$. 
\end{lem}
\begin{proof}
By the standard asymptotic expansion of MLE, we have
\begin{align*}
 \thetamle  - \thetatrue &= I_*^{-1} \dot \eta_*^T(\momentX - \momenttrue) + \Op(\dovern), \\
 \thetak  - \thetatrue &= I_*^{-1} \dot \eta_*^T(\momentXk - \momenttrue) + \Op(\dovern), ~~~~\forall k \in [d]. 
\end{align*}
The result follows directly by combining the above equations. 
\end{proof}

We now give the general definition of the statistical curvature for vector parameters.   
\begin{defn}[Statistical Curvature for Vector Parameters]
\label{def:curvaturegeneral}
Consider a curved exponential family $\mathcal{P} = \{p(x | \theta) ~\colon~ \theta \in \Theta\}$. 
Let $\Lambda$ be a $\dimtheta \times \dimtheta$ matrix whose elements are 
$$
\Lambda_{ij} = \trace( I_*^{-1} (\ddot\eta_i^*)^T N_* \Sigma_* N_*^T \ddot\eta_j^*), 
$$
where $N_* = \identitymatrix{\dimeta} -   \Sigma_{*} \dot\eta_* ( \dot\eta_*^T \Sigma_* \dot \eta)^{-1}  \dot \eta_*^T$, as defined in Lemma~\ref{lem:projection}.  
Then the statistical curvature of $\mathcal{P}$ at $\thetatrue$ is defined as 
$$
\gamma_*^2 = \trace(\Lambda I_*^{-1}). 
$$
See \citet{kass2011geometrical} for the equivalent definition with a different notation. 
\end{defn}

\section{KL-Divergence Based Combination}
We first study the KL average $\thetakl$. The following Theorem extends Theorem~\ref{thm:main}~(2) in the main paper to general vector parameters. 
\begin{thm}
\label{thm:KL_multivariate}
(1). For the curved exponential family, we have as $n\to +\infty$, 
$$
n [I_* (\thetakl - \thetamle)]_i  \dto \trace(G_i W), 
$$
where $[\cdot]_i$ denotes the $i$-th element, and $G_i$ is a $\dimeta\times\dimeta$ deterministic matrix, and $W$ is a random matrix with a Wishart distribution,  
\begin{align*}
G_i = \frac{1}{2} ({G_{i0}} + {G_{i0}^{T}}), &&
G_{i0} =  \dot\eta_* I_*^{-1}(\ddot\eta_i^*)^T N_* , && %
\text{and~~~~~~~}W  \sim \Wishart(\Sigma_*, d-1). 
\end{align*}
Here $N_*$ is defined in Lemma~\ref{lem:projection}. 

(2).  Further, we have as $n\to+\infty$, 
\begin{align*}
& n \E_{\thetatrue}[\thetakl - \thetamle] \to  0,  \\
& n^2 \E_{\thetatrue}[(\thetakl - \thetamle)(\thetakl - \thetamle)^T] \to (d-1) I_*^{-1}  \Lambda I_*^{-1}, 
\end{align*}
where $\Lambda$ is defined in Definition~\ref{def:curvaturegeneral}. 
Note that $\gamma_*^2 = \trace(\Lambda I_*^{-1}) $; this gives %
$$  n^2 \E[ || I_*^{1/2} (\thetakl - \thetamle) ||^2] \to (d-1) \gamma_*^2. $$
\end{thm}
\begin{proof}
(i). We first show that $\thetakl$ is a consistent estimator of $\thetatrue$. The proof is similar to the standard proof of the consistency of MLE on curved exponential families. We only provide a brief sketch here;  see e.g., \citet[][page 15]{ghosh1994higher} or \citet[][page 40 and Corollary 2.6.2]{kass2011geometrical} for the full technical details. 
We start by noting that $\thetaKL$ is the solution of the following equation, %
\begin{align*}
\dot\eta(\theta)^T ( \frac{1}{d}\sum_k \mu(\thetak) - \mu(\theta)) = 0. 
\end{align*}
Using the implicit function theorem, we can find an unique smooth solution $\thetakl = f_{\KL}(\thetaall)$ in a neighborhood of $\thetatrue$. In addition, it is easy to verify that $f_{\KL}(\theta, \ldots, \theta) = \theta$. The consistency of $\thetakl$ then follows by the consistency of $\thetaall$ and the continuity of $f_{\KL}(\cdot)$.  

(ii). Denote $\ell( \thetakl;  x)  = \log p(x | \thetakl)$, and let $\dot\ell( \thetakl;  x) $ and $\ddot\ell( \thetakl;  x) $ be the first and second order derivatives w.r.t. $\theta$, respectively. 
Again,  note that $\thetakl$ satisfies
\begin{align}
\sum_k  \E_{\thetak} \big[ \dot\ell(\thetakl  ;  x) \big] =  0,  %
 \label{equ:zerogradthetakl1}
\end{align}
Taking the Taylor expansion of Eq~\eqref{equ:zerogradthetakl1} around $\thetamle$, we get, 
$$
\sum_k  \E_{\thetak} \big[ \dot \ell(\thetamle ; x) + (\ddot \ell(\thetamle; x) + \op(1) ) (\thetakl - \thetamle)   \big] = 0.%
$$
Therefore, we have %
\begin{align*}
n I_{*} (\thetakl - \thetamle)  \dto  S,
&&\text{where~~~~~~~}
S  =  \frac{n}{d} \sum_k  \E_{\thetak} \big[ \dot \ell(\thetamle ; x) \big ] .
\end{align*}
We just need to show that $S_i \dto \trace(G_iW)$. 

Note the following zero-gradient equations for $\thetamle$ and $\thetak$, 
\begin{align*}
 &\E_X [\dot \ell(\thetamle ; x) ]  =  \dot\eta(\thetamle)^T( \E_{X}[\phi(x)]  - \E_{\thetamle}[\phi(x)]  ) = 0, \\
& \E_{X^k} [\dot \ell(\thetak ; x) ]  =  \dot\eta(\thetak)^T(  \E_{X^k}[\phi(x)] - \E_{\thetak}[\phi(x)]  ) = 0, \\
& \E_{\thetak} [\dot \ell(\thetak ; x) ]  =  \dot\eta(\thetak)^T(\E_{\thetak}[\phi(x)]  - \E_{\thetak}[\phi(x)] ) = 0. 
\end{align*}
We have 
\begin{align*} 
\sum_k  \E_{\thetak} \big[ \dot \ell(\thetamle ; x)  ] 
& =  \sum_k  (\E_{\thetak} - \E_{X^k}) \big[ \dot \ell(\thetamle ; x)  ]  \\
& =  \sum_k  (\E_{\thetak} - \E_{X^k}) \big[ \dot \ell(\thetamle ; x) - \dot \ell(\thetak ; x)  ]  \\
& =  \sum_k \big (\dot\eta(\thetamle) -  \dot\eta(\thetak)\big )^T \big (\E_{\thetak}[\phi(x)] - \E_{X^k}[\phi(x)] \big ), 
\end{align*}
Denote by $\hat\mu_k = \E_{X^k}[\phi(x)]$ and $\bar\mu = d^{-1}\sum_k \hat\mu_k = \E[\phi(x)]$.  
Because both $\thetamle - \thetak$ and $\E_{\thetak}[\phi(x)] - \E_{X^k}[\phi(x)]$ are $\Op(n^{-1/2})$, we have 
\begin{align*} 
S_i & =  \frac{n}{d}\sum_k \big (\dot\eta_i(\thetamle) -  \dot\eta_i(\thetak)\big )^T \big (\E_{\thetak}[\phi(x)] - \E_{X^k}[\phi(x)] \big ) \\
& =  \frac{n}{d}\sum_k \big (\thetamle  - \thetak)^T  \ddot\eta_i(\thetak)^T \big (\E_{\thetak}[\phi(x)] - \E_{X^k}[\phi(x)] \big )   + \Op(n^{-1/2}) \\
& =  \frac{n}{d}\sum_k   [ I_*^{-1}\dot\eta_*^T(\bar{\mu}  - \hat\mu_k) ]^T  \ddot\eta_i(\thetak)^T (-N_* (\hat \mu_k  - \mu_*) ) + \Op(n^{-1/2})    \text{\small ~~~~~~(By Lemma~\ref{lem:projection} and \ref{lem:k2mle})}    \\
& =  \frac{n}{d}\sum_k ( \hat\mu_k - \bar{\mu} )^T   \dot\eta_* I_*^{-1} (\ddot\eta_i^*) ^T N_* ( \hat \mu_k  - \mu_* ) +  \Op(n^{-1/2}) \\
& =  \frac{n}{d}\sum_k ( \hat\mu_k - \bar{\mu} )^T   \dot\eta_* I_*^{-1} (\ddot\eta_i^*) ^T N_* ( \hat \mu_k  - \bar\mu ) +  \Op(n^{-1/2}) \\
& =     \trace(G_{i0}W ) +   \Op(n^{-1/2}) \\
& =     \trace(G_iW ) +   \Op(n^{-1/2}) 
\end{align*} 
\begin{align*}
\text{where}~~~G_i = \frac{1}{2}(G_{i0} + G_{i0}^T)~~~, &&
G_{i0} =  \dot\eta_* I_*^{-1} (\ddot\eta_i^*) ^T N_*~~~, &&
\text{and~~~~~~~~}  W = \frac{n}{d} \sum_k (\hat\mu_k - \bar\mu) (\hat\mu_k - \bar\mu)^T.%
\end{align*}
Note that $\sqrt{\frac{n}{d}} (\hat\mu_k - \mu_*) \to\normal(0,  \Sigma_*)$, we have 
$$
 W \dto \Wishart(\Sigma_*, d-1). 
$$
This proves Part (1). 

Part (2) involves calculating the first and second order moments of Wishart distribution (see Section~\ref{sec:wishart} for an introduction of Wishart distribution). Following Lemma~\ref{lem:wishart_moments}, we have
$$
\E[\trace(G_i W )] = (d-1) \trace(G_i \Sigma_*) = 
\trace(  \dot\eta_* I_*^{-1} (\ddot\eta_i^*) ^T N_* \Sigma_*) =
\trace(   I_*^{-1} (\ddot\eta_i^*) ^T N_* \Sigma_* \dot\eta_*) = 0 , %
$$
where we used the fact that $N_*\Sigma_* \dot\eta_* = 0$ as shown in Lemma~\ref{lem:projection}. For the second order moments, 
\begin{align*}
&\E[\trace(G_iW )\trace(G_j W)]  \\
& ~~~~ = 2(d -1) \trace(G_i\Sigma_* G_j \Sigma_*) + (d-1)^2 \trace(G_i\Sigma_*)\trace(G_j\Sigma_*)    ~~~~~~~~~~\text{\small (By Lemma~\ref{lem:wishart_moments})} \\
&~~~~=2(d-1)  \trace(G_i\Sigma_* G_j \Sigma_*)   + 0\\
&~~~~=(d-1)\big [ \trace(G_{i0} \Sigma_* G_{j0} \Sigma_*) + \trace(G_{i0} \Sigma_* G_{j0}^{T} \Sigma_*)  \big]  ~~~~~~~~~~\text{\small (Recall $G_i =(G_{i0} + G_{i0}^T)/2$)},
\end{align*}
for which we can show that 
\begin{align*}
\trace(G_{i0} \Sigma_* G_{j0} \Sigma_* ) 
&= \trace(  \dot\eta_* I_*^{-1} (\ddot\eta_i^*) ^T N_* \Sigma_*    \dot\eta_* I_*^{-1} (\ddot\eta_j^*) ^T N_* \Sigma_*) \\
&= \trace( I_*^{-1} (\ddot\eta_i^*) ^T N_* \Sigma_*    \dot\eta_* I_*^{-1} (\ddot\eta_j^*) ^T N_* \Sigma_*   \dot\eta_*)  \\
& = 0, 
\end{align*}
and 
\begin{align*}
\trace(G_{i0} \Sigma_* G_{j0}^T \Sigma_* ) 
&= \trace(  \dot\eta_* I_*^{-1} (\ddot\eta_i^*) ^T N_* \Sigma_*   N_*^T \ddot\eta_j^*  I_*^{-1} \dot\eta_* ^T  \Sigma_*) \\
&= \trace(  I_*^{-1} (\ddot\eta_i^*) ^T N_* \Sigma_*   N_*^T \ddot\eta_j^*  I_*^{-1} \dot\eta_* ^T  \Sigma_* \dot\eta_* ) \\
&= \trace(  I_*^{-1} (\ddot\eta_i^*) ^T N_* \Sigma_*   N_*^T \ddot\eta_j^* ) = \Lambda_{ij} .
\end{align*}
This finishes the proof. 
\end{proof}

\section{Linear Combination}

We analyze the linear combination method in this section. The following theorem generalizes the results in Theorem~\ref{thm:main} (3). 

\begin{thm}
\label{thm:linear_tensor}
(1). For curved exponential families, we have
\begin{align}
n [I_*(\thetalinear - \thetamle) ]_i  \dto \trace( L_i W), %
\end{align}
where $[\cdot]_i$ represents the $i$-th element, 
and $L_i$ (as defined in Lemma~\ref{lem:higherMLE}) is a deterministc matrix and $W$ is a random Wishart matrix, 
\begin{align}
 L_i = \frac{1}{2}(\ddot\eta_i I_*^{-1} \dot\eta_*^T + \dot\eta_* I_*^{-1} \ddot\eta_i^T)  +   \frac{1}{2}  \dot\eta_*  I_*^{-1} J_i I_*^{-1} \dot\eta_*^T,  &&
 W \sim \Wishart(\Sigma_*, d-1).
 \end{align}
 (2). Further, we have
\begin{align*}
&n \E_{\thetatrue}[I_*(\thetalinear  - \thetamle)]_i \to  (d-1) \trace (B_i), \\
&n^2 \E_{\thetatrue}[(\thetalinear - \thetamle) (\thetalinear - \thetamle)^T] \to (d-1)I_*^{-1} (\Lambda  + D )I_*^{-1} , 
\end{align*}
where $B_i = I_*^{-1} (\dot\eta_*^T  \Sigma_* \ddot\eta_i^*  +   \frac{1}{2} J_i ) $ and $D$ is a semi-definite matrix whose $(i,j)$-element is 
$$
D_{ij} = 2 \trace(B_i B_j)  + (d-1) \trace(B_i)\trace(B_j). 
$$
\end{thm}
\newcommand{\noverd}{n}%
\begin{proof}
Denote by $\hat\mu_k = \E_{X^k}[\phi(x)]$, and $\bar\mu = d^{-1}\sum_{k} \hat\mu_k = \E_{X}[\phi(x)]$. 
Using the second order expansion of MLE in Lemma~\ref{lem:higherMLE}, we have
\begin{align*}
& [I_*(\thetamle - \thetatrue) - \dot\eta_*^T(\bar\mu - \mu_*)]_i   =  (\bar\mu - \mu_*)^T L_i (\bar\mu - \mu_*)   + \Op(n^{-3/2}) , \\
 &[I_*(\thetak - \thetatrue) - \dot\eta_*^T(
 \hat\mu_k - \mu_*)]_i   = (\hat\mu_k - \mu_*) ^T L_i (\hat\mu_k - \mu_*)   + \Op(\noverd^{-3/2}) ,~~~~\forall k \in [d]. 
 \end{align*}
Note that $\thetalinear = \sum_k \thetak/d$. Combine the above equations, we get,
\begin{align}
n  [I_*(\thetalinear - \thetamle) ]_i   
  & =\frac{n}{d}\sum_k (\hat\mu_k - \bar\mu) ^T L_i (\hat\mu_k -\bar \mu)   + \Op(n^{-1/2}) , \\
 & =  \trace( L_i W) + \Op(n^{-1/2}), %
\end{align}
where (because $\sqrt{\frac{n}{d}} (\hat\mu_k - \mu_*) \dto \normal(0, \Sigma_*)$) 
$$
W = \frac{n}{d}\sum_k (\hat\mu_k - \bar\mu)  (\hat\mu_k -\bar \mu)^T ~ \dto ~ \Wishart(\Sigma_*, d-1).
$$
This finishes the proof of Part (1). 

Part (2) involves calculating the first and second order moments. 
For the first order moments,  
\begin{align*}
\E[ \trace(L_i W)] 
& =  (d-1) \trace(L_i \Sigma_*) \\
&= (d-1) \trace[ ( \ddot\eta_i^* I_*^{-1} \dot\eta_*^T  +   \frac{1}{2}  \dot\eta_*  I_*^{-1} J_i I_*^{-1} \dot\eta_*^T) \Sigma_*] \\
& = (d-1)  \trace[  I_*^{-1} \dot\eta_*^T  \Sigma_* \ddot\eta_i^*  +   \frac{1}{2}   I_*^{-1} J_i I_*^{-1} \dot\eta_*^T \Sigma_*  \dot\eta_*]\\
& = (d-1)  \trace[  I_*^{-1} (\dot\eta_*^T  \Sigma_* \ddot\eta_i^*  +   \frac{1}{2} J_i ) ]  \\
& = (d-1) \trace (B_i). %
\end{align*}

Denote by $L_{i0}=\ddot\eta_i^* I_*^{-1} \dot\eta_*^T  +   \frac{1}{2}  \dot\eta_*  I_*^{-1} J_i I_*^{-1} \dot\eta_*^T$, 
then we have $L_i = \frac{1}{2}(L_{i0} + L_{i0}^T)$. For the second order moments, we have
$$
\trace(L_i \Sigma_* L_j \Sigma_*)  = 
\frac{1}{2}( \trace(L_{i0} \Sigma_* L_{j0} \Sigma_*) + \trace(L_{i0} \Sigma_* L_{j0}^T \Sigma_*) )
$$
where 
\begin{align*}
\trace(L_{i0}  \Sigma_* L_{j0}  \Sigma_*) 
&=  \trace(  I^{-1}  ( \dot\eta_*^T \Sigma_* \ddot\eta_j^*  + \frac{1}{2}J_j) I^{-1}    ( \dot\eta_*^T \Sigma_* \ddot\eta_i^*  + \frac{1}{2}J_i)  ) \\
&= \trace(B_i B_j). \\
\end{align*}
and 
\begin{align*}
\trace(L_{i0}  \Sigma_* L_{j0}^T  \Sigma_*) 
&= \trace(B_i B_j)  +  \trace(I_*^{-1} (\ddot\eta_j^*\Sigma_* \ddot\eta_i^* - \dot\eta_* \Sigma_* \ddot\eta_j^* I^{-1} \dot\eta_* \Sigma_* \ddot\eta_i^* ))\\
&= \trace(B_i B_j) + \trace(I_*^{-1}  (\ddot\eta_i^*)^T N_* \Sigma_* N_*^T \ddot\eta_j^* )\\
&= \trace(B_i B_j)  + \Lambda_{ij},
\end{align*}
where $\Lambda_{ij}$ is defined in Theorem~\ref{thm:KL_multivariate}. 
So we have
$$\trace(L_i \Sigma_* L_j \Sigma_*)   = \trace(B_i B_j)  +\frac{1}{2} \Lambda_{ij},$$
and hence
\begin{align*}
\E[\trace(L_iW) \trace(L_j W) ] & =  2(d-1) \trace(L_i \Sigma_* L_j \Sigma_* ) + (d-1)^2\trace(L_i\Sigma_*) \trace(L_j \Sigma_*)  \\
& = (d-1) \Lambda_{ij}  +  2(d-1) \trace(B_i B_j)  + (d-1)^2 \trace(B_i)\trace(B_j). 
\end{align*}
So 
\begin{align*}
n \E[(\thetalinear - \thetamle) (\thetalinear - \thetamle)^T] \to (d-1)I_*^{-1} (\Lambda  + D )I_*^{-1} . 
\end{align*}
This finishes the proof. 
\end{proof}

\section{General Consistent Combination}%

We now consider general consistent combination functions $\thetaf= f(\thetahat^1, \ldots, \thetahat^d)$.  
To start, we show that $f(\cdot)$ can be assumed to be symmetric without loss of generality: If $f(\cdot)$ is not symmetric, one can construct a symmetric function that performs no worse than $f(\cdot)$. %
\begin{lem}
For any combination function $\thetaf = f(\thetahat^1, \ldots , \thetahat^d)$, define a symmetric function via 
$$
 \thetahat^{\bar f} = \bar f(\thetahat^1, \ldots, \thetahat^d) = \frac{1}{d!} \sum_{\sigma \in \Gamma} f(\thetahat^{\sigma(1)}, \ldots, \thetahat^{\sigma(d)}),  
~~~~~~~\text{where $\Gamma$ is the set of permutations on $[d]$.}
$$
then we have
\begin{align}
\E_{\thetatrue}[\thetahat^{\bar f} - \thetamle] = \E_{\thetatrue}[\thetaf - \thetamle] , &&\text{and} &&
 \E_{\thetatrue}[||\thetahat^{\bar f} - \thetamle||^2]  \leq \E_{\thetatrue}[||\thetaf - \thetamle||^2] , 
\end{align}
which is also true if $\thetamle$ is replaced by $\thetatrue$. 
\end{lem}
\begin{proof}
Because $X^k, k \in [d]$ are i.i.d. sub-samples, the expected bias and MSE of $\thetaf$ would not change if we work on a permuted version $X^{\sigma(k)}, k \in [d]$, where $\sigma$ is any permutation on $[d]$. This implies
\begin{align}
\E_{\thetatrue}[\thetahat^{\sigma(f)} - \thetamle] = \E_{\thetatrue}[\thetaf - \thetamle] , && 
 \E_{\thetatrue}[||\thetahat^{\sigma(f)} - \thetamle||^2]  = \E_{\thetatrue}[||\thetaf - \thetamle||^2] , 
\end{align}
where $\thetahat^{\sigma(f)} =f(\thetahat^{\sigma(1)}, \ldots, \thetahat^{\sigma(d)})$. The result then follows straightforwardly, 
\begin{align*}
&\E_{\thetatrue}[\thetahat^{\bar f} - \thetamle] = \frac{1}{d!}\sum_{\sigma}  \E_{\thetatrue}[\thetahat^{\sigma(f)} - \thetamle] 
= \E_{\thetatrue}[\thetaf - \thetamle] , \\
&\E_{\thetatrue}[||\thetahat^{\bar f} - \thetamle||]  \leq \frac{1}{d!}\sum_{\sigma}  \E_{\thetatrue}[ || \thetahat^{\sigma(f)} - \thetamle||^2] 
= \E_{\thetatrue}[||\thetaf - \thetamle||^2] .
\end{align*}
This concludes the proof. 
\end{proof}
We need to introduce some derivative notations before presenting the main result. Assuming $f(\cdot)$ is differentiable, we write
\begin{align*}
 \partial_k^i f (\theta) = \left. \myp{f_i(\theta_1, \ldots, \theta_d)}{\theta_k} \right |_{ \begin{subarray}{l}  \theta_k = \theta \\ \forall k \in [d]  \end{subarray}}   ~~\text{and} &&
\partial_{kl}^i f(\theta) = \left. \myp{^2f_i(\theta_1, \ldots, \theta_d)}{\theta_k \partial\theta_l} \right |_{ \begin{subarray}{l}  \theta_k = \theta \\ \forall k \in [d]  \end{subarray}}~ ,  && %
\forall \theta \in \Theta,  ~~~k, l\in [d]. 
\end{align*}
Since $f(\cdot)$ is symmetric, we have 
$\partial^i_{11} f(\theta) = \partial_{kk}^i f(\theta)$, and 
$\partial^i_{12} f(\theta) = \partial_{kl}^i f(\theta)$ for $\forall k, l \in [d]$. 

\begin{thm}
(1). Consider a consistent and symmetric $\thetaf = f(\thetahat^1, \ldots, \thetahat^d)$ as in Definition~\ref{def:consistent}. Assume its first three order derivatives exist, 
then we have as $n\to +\infty$, 
$$n [I_*(\thetaf - \thetamle) ]_i \to \trace(F_i W), $$
where $[\cdot]_i$ denotes the $i$-th element, and $F_i$ is a deterministic matrix and $W$ is a random matrix, %
\begin{align*}
F_i %
=  \frac{1}{2}( \ddot\eta_iI_*^{-1} \dot\eta_*^T  + \dot\eta_* I_*^{-1} \ddot\eta_i^T  )
+ \frac{1}{2}\dot\eta_* I_*^{-1} (J_i +  d (\partial^i_{11} f_*  - \partial^i_{12} f_* ) )  I_*^{-1} \dot\eta_*^T, && 
W \sim \Wishart(\Sigma_*, d-1).
\end{align*}

(2). Further, we have
\begin{align*}
&n \E_{\thetatrue}[I_*(\thetalinear  - \thetamle)]_i \to  (d-1) \trace (B_i), \\
&n^2 \E_{\thetatrue}[(\thetalinear - \thetamle) (\thetalinear - \thetamle)^T] \to (d-1)I_*^{-1} (\Lambda  + D )I_*^{-1} 
\end{align*}
where $B_i = I_*^{-1} (\dot\eta_*^T  \Sigma_* \ddot\eta_i^*  +   \frac{1}{2} (J_i + d (\partial^i_{11} f_* - \partial^i_{12} f_* ) ) ) $ and $D$ is a semi-definite matrix whose $(i,j)$-element is 
$$
D_{ij} = 2 \trace(B_i B_j)  + (d-1) \trace(B_i)\trace(B_j). 
$$
\end{thm}
\begin{proof}
By the consistency and the continuity of $f(\cdot)$, we have $f(\theta, \ldots, \theta) = \theta$. Taking the derivative on the both side, we get 
$ \sum_k \partial_k^i f (\theta) = 1$. Since $f(\cdot)$ is symmetric, we get $\partial_k^i f(\theta) = 1/d $, $\forall k\in [d]$. 

Expanding $\thetaf = f(\thetahat^1, \ldots, \thetahat^d)$ around $\thetamle$, we get 
\begin{align*}
[\thetaf - \thetamle]_i 
& = [f(\thetahat^1, \ldots, \thetahat^d) - f(\thetamle, \ldots, \thetamle)]_i \\
& =  [\thetalinear - \thetamle]_i  + \frac{1}{2} \sum_{k,l}(\thetak - \thetamle)^T \partial^i_{kl} f(\thetamle) (\thetahat^l - \thetamle) 
 + \Op(n^{-3/2})\\
& =  [\thetalinear - \thetamle]_i  + \frac{1}{2} \sum_{k,l}(\thetak - \thetamle)^T \partial^i_{kl} f_* (\thetahat^l - \thetamle) 
 + \Op(n^{-3/2}),
\end{align*}
where the second term is %
\begin{align*}
&   \sum_{k,l}(\thetak - \thetamle)^T \partial^i_{kl}  f_* (\thetahat^l - \thetamle)  \\
& = \sum_{k}(\thetak - \thetamle)^T (\partial^i_{11}f_* -  \partial^i_{12} f_*)  (\thetahat^k - \thetamle) + 
 d^2(\thetalinear - \thetamle)^T   \partial^i_{12}f_*   (\thetalinear - \thetamle)  \\
 & =  \sum_{k}(\thetak - \thetamle)^T (\partial^i_{11}f_* -  \partial^i_{12} f_*)  (\thetahat^k - \thetamle)  + 
\Op(n^{-2}) %
 ~~~~~~~~~~\text{\small (since $\thetalinear- \thetamle = \Op(n^{-1})$)} \\
  & = \sum_{k} (\hat\mu_k - \bar\mu)^T \dot\eta_* I_*^{-1}  (\partial^i_{11}f_* -  \partial^i_{12} f_*)  I_*^{-1} \dot\eta_*^T (\hat\mu_k - \bar\mu) +\Op(n^{-3/2})
~~~~~~~\text{\small (by Lemma~\ref{lem:k2mle})}  
\\
  & = \frac{d}{n} \trace( \dot\eta_* I_*^{-1}   (\partial f^i_{11}- \partial f^i_{12})  I_*^{-1} \dot\eta_*^T 
 W)+ \Op(n^{-3/2}). %
\end{align*}
where $W =  \frac{n}{d} \sum_{k} (\mu_k - \bar\mu) (\mu_k - \bar\mu)^T  $. 
Combined with Theorem~\ref{thm:linear_tensor}, we get
$$
n [\thetaf - \thetamle]_i  = 
\trace(F_i W), 
$$
where%
$$
F_i%
=  \frac{1}{2}( \ddot\eta_iI_*^{-1} \dot\eta_*^T  + \dot\eta_* I_*^{-1} \ddot\eta_i^T  )
+ \frac{1}{2}
\dot\eta_* I_*^{-1} (J_i +  d (\partial f^i_{11}- \partial f^i_{12}))  I_*^{-1} \dot\eta_*^T. 
$$
The proof of part (2) is similar to that of Theorem~\ref{thm:linear_tensor}. 
\end{proof}

\section{Proof of Theorem~\ref{thm:main}~(5)}
We have mainly focused on the MSE w.r.t. the global MLE $\E_{\thetatrue}[||\thetaf-\thetamle||^2]$. The results can be conveniently related to the MSE w.r.t. the true parameter $\E[||\thetaf - \thetatrue||^2]$ via the following Lemma. %
\begin{lem}%
\label{lem:MLEorth}
For any first order efficient estimator $\thetahat$, we have
$$
\cov(\thetahat - \thetamle,   \thetamle) = \oo(n^{-2}), 
$$
where $\cov(\cdot)$ denotes the covariance,  $\cov(x, y) = \E[(x-\E(x))^T (y- \E(y))]$. 
This suggests that the residual $\thetahat - \thetamle$ is orthogonal to $\thetamle$ upto $\oo(n^{-2})$. 
\end{lem}
\begin{proof}
See \citet{ghosh1994higher}, page 27. 
\end{proof}

We now ready to prove Theorem~\ref{thm:main}~(5). 
\begin{proof}[Proof of Theorem~\ref{thm:main}~(5)]
Theorem~\ref{thm:main}~(2) shows that 
$$
\var(\thetakl - \thetamle) =  \E_{\thetatrue}[||\thetakl - \thetamle||^2] -  ||\E_{\thetatrue}(\thetakl-\thetamle)||^2 = \frac{d-1}{n^2}\gamma_*^2 I_*^{-1} + \oo(n^{-2}). 
$$
In addition, note that $\E_{\thetatrue}(\thetakl - \thetatrue) = O(n^{-1})$, and
$\E_{\thetatrue}(\thetamle - \thetatrue) = O(n^{-1})$, and 
$\E_{\thetatrue}(\thetakl - \thetamle) = o(n^{-1})$, we have 
\begin{align*}
  ||\E_{\thetatrue}(\thetakl - \thetatrue)||^2 - ||\E_{\thetatrue}(\thetamle - \thetatrue)||^2 
   =  \E_{\thetatrue}(\thetakl  - \thetatrue + \thetamle - \thetatrue) \E_{\thetatrue}(\thetamle - \thetakl)^T
   = o(n^{-2}). 
\end{align*}
Denote the variance by $\var(\thetahat) = \E_{\thetatrue}[|| \theta - \E_{\thetatrue}(\theta)||^2] $, we have
\begin{align*}
\E_{\thetatrue}[||\thetakl - \thetatrue||^2]  
& = \var(\thetakl) +  ||\E_{\thetatrue}(\thetakl - \thetatrue)||^2\\
& = \var(\thetamle)  + \var(\thetakl - \thetamle) +  ||\E_{\thetatrue}(\thetakl - \thetatrue)||^2 ~~~~~~~~~~~~~~~ \text{\small (by Lemma~\ref{lem:MLEorth})}\\  
&  =  \E_{\thetatrue}[||\thetamle-\thetatrue||^2]  - ||\E_{\thetatrue}(\thetamle - \thetatrue)||^2 + \var(\thetakl - \thetamle)  +  ||\E_{\thetatrue}(\thetakl - \thetatrue)||^2\\
&  =  \E_{\thetatrue}[||\thetamle-\thetatrue||^2] + \frac{d-1}{n^2}\gamma_*^2 I_*^{-1}  + \oo(n^{-2}). 
\end{align*}
The result for $\E [|| \thetalinear - \thetatrue ||^2]$ can be shown in a similar way. 
\end{proof}

\section{Lower Bound} 

We prove the asymptotic lower bound in Theorem~\ref{thm:lowerbound}. 

\begin{thm}
Assume $ \thetaf = f(\thetahat^1, \ldots , \thetahat^d)$ is any measurable function. We have
$$
\liminf_{n \to \infty}  n^2 ~ \E_{\thetatrue} [ || I_*^{1/2} (\thetaf - \thetamle) ||^2 ] \geq  (d-1)\gamma_*^2. 
$$
\end{thm}
\begin{proof}
Define $\thetaproj = \E_{\thetatrue}(\thetamle | \thetahat^1, \ldots, \thetahat^d)$, then we have for any $f(\cdot)$, 
\begin{align}
n^2 \E_{\thetatrue} [ || I_*^{1/2}(\thetaf - \thetamle) ||^2 ] 
&  =n^2 \E_{\thetatrue} [ || I_*^{1/2}(\thetaf - \thetaproj) ||^2 ]  + 
n^2\E_{\thetatrue} [ || I_*^{1/2}(\thetaproj - \thetamle) ||^2 ]  \label{equ:triangle} \\ 
& \geq n^2\E_{\thetatrue} [ || I_*^{1/2}(\thetaproj - \thetamle) ||^2 ]  \notag \\
& = n^2 \E_{\thetatrue} [ || I_*^{1/2}(\thetamle - \E_{\thetatrue}(\thetamle | \thetaall)) ||^2 ] \notag \\
& \overset{def}{=} n^2 \E_{\thetatrue} [ \var(I_*^{1/2}\thetamle | \thetaall))] .  \notag
\end{align}
Therefore, $\thetaproj$ is the projection of $\thetamle$ onto the set of random variables in the form of $f(\thetahat^1, \ldots, \thetahat^d)$, and forms the best possible combination. Applying \eqref{equ:triangle} to $\thetakl$, we get%
\begin{align*}
n^2 \E_{\thetatrue} [ \var( I_*^{1/2}\thetamle | \thetaall))]  
 & = n^2 \E_{\thetatrue} [ ||  I_*^{1/2}(\thetakl- \thetamle) ||^2 ]  -  \E_{\thetatrue} [ || n I_*^{1/2}(\thetakl - \thetaproj) ||^2 ]    . %
\end{align*}
Since we have shown that $n^2\E_{\thetatrue} [ || I_*^{1/2}(\thetakl- \thetamle) ||^2 ]   \to (d-1)\gamma_*^2$ in Theorem~\ref{thm:KL_multivariate}, 
the result would follows if we can show that 
$\thetakl - \thetaproj = \op(n^{-1})$,  
that is, $\thetakl$ is equivalent to $\thetaproj$ (upto $\op(n^{-1})$). 
To this end, note that by following the proof of Theorem~\ref{thm:KL_multivariate}, we have
\begin{align*}
[n I_*(\thetakl - \thetamle)]_i = \trace(G_i W) + \op(1), %
\end{align*}
where $G_i$ is defined in Theorem~\ref{thm:KL_multivariate} and
$W = \frac{n}{d}\sum_k (\hat\mu_k - \mu_*)  (\hat\mu_k - \mu_*)^T$. 
Combining this with Lemma~\ref{lem:condexptrGW} below, we get 
\begin{align*}
n I_*(\thetakl - \thetaproj) = 
\E ( n I_*(\thetakl - \thetamle)  ~|~ \thetaall) = \op(1). 
\end{align*}
This finishes the proof. 
\end{proof}

\begin{lem}
\label{lem:condexptrGW}
Let $G_i$ be defined in Theorem~\ref{thm:KL_multivariate} and $W =  \frac{n}{d}\sum_k (\hat\mu_k - \bar \mu) (\hat\mu_k - \bar \mu)^T$, where $\hat\mu_k = \E_{X^k} \phi(x)$ and $\bar\mu = \frac{1}{d} \hat\mu_k$. Then 
$$\E(\trace(G_i W ) | \thetaall) =  \op(1).$$%
\end{lem}
\begin{proof} 
Let $z_k = \sqrt\frac{n}{d} (\hat\mu_k -  \mu_*)$ and $\bar z = \frac{1}{d} \sum_k z_k$, then $W =\sum_k (z_k-\bar z) (z_k- \bar z)^T$. 
By Lemma~\ref{lem:condmu}, we have $\E(z_k z_k^T | \thetak ) \pto N_{*} \Sigma_* N_*^T$. This gives, 
$$
\E(W | \thetaall)  
~=~ \frac{d-1}{d} \sum_k \E(z_k z_k^T | \thetak ) 
~\pto~  (d-1)  N_* \Sigma_* N_*^T,  %
$$
and therefore, 
\begin{align*}
\E(\trace(G_i W ) | \thetaall) 
& \pto (d-1) \trace(G_i  N_* \Sigma_* N_*^T)  \\
& = (d-1) \trace(  \dot\eta_* I_*^{-1} (\ddot\eta^*_i)^T N_*  N_* \Sigma_* N_*^T)  \\
& = (d-1) \trace( I_*^{-1} (\ddot\eta^*_i)^T N_*  N_* \Sigma_* N_*^T  \dot\eta_* ) \\
& = 0
\end{align*}
where the last step used $N_*^T\dot\eta_*=0$ in Lemma~\ref{lem:projection}. 
\end{proof}

\begin{lem}
\label{lem:condmu}
Assume $\thetamle$ is the maximum likelihood estimate on sample $X$ of size $n$. Then we have 
$$
\Cov( \sqrt{n} ( \mu_X - \mu_*)  ~|~ \thetamle) \pto  N_{*}\Sigma_* N_*^T,   ~~~~~~ \text{as $n\to +\infty$}, 
$$
where $N_{*} = \identitymatrix{\dimeta}  - \Sigma_* \dot\eta_* (\dot\eta_*^T \Sigma_* \dot\eta_*)^{-1}  \dot\eta_*^T$ as defined in Lemma~\ref{lem:projection}.
\end{lem}
\begin{proof}
Define $z_X =  \sqrt{n}(\mu_X - \mu_*)$,  
$z_{\thetamle} = \sqrt{n}(\mu_{\thetamle} - \mu_*)$ and $z_{\perp} = z_X - z_{\thetamle}$. Then by Lemma~\ref{lem:projection}, we have $z_{\thetamle} = P_* z_X + \op(1)$ and $z_{\perp} = N_* z_X + \op(1)$.
 Because $z \dto \normal(0, \Sigma_*) $, we have %
 $$
 \begin{bmatrix} z_{\thetamle}\\  z_{\perp}\end{bmatrix} ~~ \dto ~~\normal(0,~~
 \begin{bmatrix} P_*\Sigma_* P_*^T & 0 \\ 0  &  N_*\Sigma_*N_*^T \end{bmatrix} ), 
  $$
  where we used the fact that $N_* \Sigma_* P_*^T = 0$. 
 Therefore,
 \begin{align} 
 \Cov(z | \thetamle) = \Cov(z | z_{\thetamle}) = \Cov(z_{\thetamle} + z_{\perp} | z_{\thetamle} ) 
 = \Cov( z_{\perp} | z_{\thetamle} ) 
 \pto  N_*\Sigma_*N_*^T, %
\end{align}
where we assumed that the convergence of the joint distribution implies the convergence of the conditional moment; see e.g., \citet{steck1957limit, sweeting1989conditional} for  technical discussions on conditional limit theorems. 
\end{proof}

\section{Moments of Wishart Distribution}
\label{sec:wishart}

 In this section we introduce a lemma about the moments of Wishart distribution that we use in our proof. 
The Wishart distribution arises as the distribution of the empirical covariance matrix of multivariate normal distributions. To be specific, assume $\{x^k \colon k \in [d]\}$ is drawn i.i.d from multivariate normal distribution $\normal(\mu, \Sigma)$, then the empirical covariance matrix is shown to follow a Wishart distribution with $(d-1)$ degrees of freedom, that is, 
\begin{align*}
\sum_{k=1}^d (x^k -  \bar x) (x^k - \bar x)^T \sim \Wishart(\Sigma, d-1),  &&  \text{where~~}  \bar x =  \frac{1}{d}\sum_k x^k. 
\end{align*}

We used the following result in our proof. 
\begin{lem}\label{lem:wishart_moments}
Assume $W \sim \Wishart(\Sigma, d)$, and $A$ and $B$ are two deterministic symmetric matrices both of the same sizes as $W$.  We have  %
\begin{align*}
& \E[\tr(A W) ]=   d\trace(A \Sigma) , \\
& \E[(\tr(A W) )^2]=    2 d \trace((A \Sigma)^2) + d^2 (\trace(A \Sigma))^2,\\
& \E[\trace(AW ) \trace(BW)] = 
2 d \trace(A\Sigma B \Sigma)  + d^2 \trace(A\Sigma) \trace(B\Sigma). 
\end{align*}
\end{lem}
\begin{proof}
 The moment generating function of Wishart distribution is \citep[see][]{muirhead2009aspects}
$$\E[\exp(\trace(A W))]   = \det(I - 2 A \Sigma)^{-d/2}.$$
Expanding the right hand size, we have, 
\begin{align*}
\E[\exp(\trace(\Theta W))]      
& = \det(I - 2 A \Sigma)^{-d/2}  \\ %
& = \exp(  - \frac{d}{2}\trace( \log(I - 2A \Sigma) )    ) \\
& = \exp(-\frac{d}{2}  ( \trace(-2A \Sigma) - \frac{1}{2}\trace((2A \Sigma)^2) + \cdots   )  ) \\
& = 1 -\frac{d}{2}  ( \trace(-2A \Sigma) - \frac{1}{2}\trace((2A \Sigma)^2)   )  
+ \frac{d^2}{8}  ( \trace(-2A \Sigma) - \frac{1}{2}\trace((2A \Sigma)^2)  )^2  + \cdots \\
& = 1 + d\trace(A \Sigma) + d \trace((A \Sigma)^2) + \frac{d^2}{2} (\trace(A \Sigma))^2 + \cdots  
\end{align*}
On the other hand, for the left hand size, we have
$$
\E[\exp(\trace(A W))]  =1  +  \E [\trace(A W)] + \frac{1}{2} \E[ (\trace(A W))^2] + \cdots 
$$
This gives 
\begin{align*}
 E[\tr(A W) ]=   d\trace(A \Sigma) , 
&& E[(\tr(A W) )^2]=   
2 d \trace((A \Sigma)^2) + d^2 (\trace(A \Sigma))^2 
\end{align*}
Finally, note that
\begin{align*}
2\trace(AW)\trace(BW) = [\trace((A+B)W)]^2   - [\trace(AW)]^2 - [\trace(BW) ]^2. 
\end{align*}
This completes the proof. 
\end{proof}

\end{document}

%% file: newcommands.tex
\usepackage{amsthm}   
\newtheorem{thm}{Theorem}[section]

\newtheorem{defn}[thm]{Definition}
\newtheorem{exa}[thm]{Example}
\newtheorem{lem}[thm]{Lemma} 

\begingroup
    \makeatletter
    \@for\theoremstyle:=definition,remark,plain\do{%
        \expandafter\g@addto@macro\csname th@\theoremstyle\endcsname{%
            \addtolength\thm@preskip\parskip
            }%
        }
\endgroup

\newcommand{\Rmnum}[1]{\expandafter\@slowromancap\romannumeral #1@}

\newcommand{\mycomment}[1]{}

 \newcommand{\myp}[2]{\frac{\displaystyle \partial #1}{\displaystyle \partial  #2}}

\newcommand{\E}{\mathbb{E}}

\DeclareMathOperator*{\argmax}{arg\,max}
\DeclareMathOperator*{\argmin}{arg\,min}

\newcommand{\Fisher}{\mathcal{I}}
\newcommand{\fisher}{I}

\newcommand{\OO}{O}
\newcommand{\oo}{o}
\newcommand{\Op}{\OO_p}
\newcommand{\op}{\oo_p}

\newcommand{\tr}{\mathrm{tr}} \newcommand{\trace}{\mathrm{tr}}
\newcommand{\cov}{\mathrm{cov}} \newcommand{\Cov}{\mathrm{Cov}}

\newcommand{\var}{\mathrm{var}}

\newcommand{\KL}{\mathrm{KL}}

\newcommand{\mle}{\mathrm{mle}}
\newcommand{\linear}{\mathrm{linear}}

\newcommand{\thetaf}{{\htheta^{f}}}
\newcommand{\thetatrue}{{\theta^*}}
\newcommand{\thetahat}{\hat{\theta}}
\newcommand{\htheta}{\hat{\theta}}
\newcommand{\thetak}{{\htheta^k}}
\newcommand{\thetamle}{{\htheta^{\mathrm{\mle}}}}
\newcommand{\thetalinear}{{\htheta^{\mathrm{\linear}}}}
\newcommand{\thetaKL}{{\htheta^{\KL}}}
\newcommand{\thetakl}{{\htheta^{\KL}}}
\newcommand{\Etrue}{\mathbb{E_{\thetatrue}}}
\newcommand{\thetaproj}{{\htheta^{\mathrm{proj}}}}
\newcommand{\thetaall}{{\thetahat^1, \ldots, \thetahat^d}}

\newcommand{\moment}{\mu}
\newcommand{\hatmoment}{{\hat\mu}}

\newcommand{\momentX}{{\hatmoment_{X}}}
\newcommand{\momentmle}{{\hatmoment_{\mle}}}
\newcommand{\momentXk}{{\hatmoment_{X^k}}}

\newcommand{\momenttrue}{{\mu_{*}}}
\newcommand{\dotmomenttrue}{{\dot\mu_{*}}}

\newcommand{\density}[2]{p(#1 | #2)}

\newcommand{\blue}[1]{\textcolor{blue}{#1}}

\newcommand{\R}{\mathbb{R}}

\newcommand{\Wishart}{\mathrm{Wishart}}

\newcommand{\normal}{\mathcal{N}}

\newcommand{\dto}{\overset{d}{\to}}
\newcommand{\pto}{\overset{p}{\to}}

\newcommand{\dimtheta}{q}
\newcommand{\dimeta}{m}

\newcommand{\logZ}{{\log Z}}

\newcommand{\identitymatrix}[1]{{\mathbf{1}_{#1\times#1}}}